\documentclass{article}
\usepackage{etoolbox}

\usepackage{ijcai19}

\usepackage{times}
\usepackage{soul}
\usepackage{url}
\usepackage[hidelinks]{hyperref}
\usepackage[utf8]{inputenc}
\usepackage[small]{caption}
\usepackage{graphicx}
\usepackage{amsmath}
\usepackage{booktabs}
\usepackage{cancel}

\usepackage{algorithm}
\usepackage{algcompatible}
\usepackage[noend]{algpseudocode}

\usepackage[utf8]{inputenc} 
\usepackage[T1]{fontenc}    
\usepackage{hyperref}       
\usepackage{url}            
\usepackage{booktabs}       
\usepackage{amsfonts}       
\usepackage{nicefrac}       
\usepackage{microtype}      

\newcommand{\figref}[1]{Fig.~\ref{#1}}
\newcommand{\algAgn}{\textsc{Agnostic}\xspace}
\newcommand{\algOmn}{\textsc{OmniTeacher}\xspace}
\newcommand{\algBbox}{\textsc{BboxTeacher}\xspace}

\usepackage{float}
\usepackage{times}
\usepackage{hyperref}
\usepackage{mathtools}
\usepackage{multirow,color,graphicx}
\usepackage{subcaption}
\usepackage{caption}
\usepackage{xfrac,amsfonts,amsbsy}
\usepackage[titletoc]{appendix}
\usepackage{xspace}
\usepackage{savesym,verbatim}
\usepackage[titletoc]{appendix}
\usepackage{paralist}
\usepackage{subscript}
\usepackage{tikz}
\usepackage{verbatim}
\usepackage{amsthm}

\usepackage{bbm}

\newtheorem{lemma}{Lemma}

\newtheorem{theorem}{Theorem}
\newtheorem{proposition}{Proposition}

\makeatletter

\newcommand{\Rmnum}[1]{\expandafter\@slowromancap\romannumeral #1@}
\makeatother

\newcommand{\citet}[1]{\citeauthor{#1} (\citeyear{#1})}

\DeclareMathOperator*{\argmin}{arg\,min}
\DeclareMathOperator*{\argmax}{arg\,max}

\newcommand{\R}{\mathbb{R}}

\newcommand{\abs}[1]{\left\vert#1\right\vert}
\def \argmax {\mathop{\rm arg\,max}}
\def \argmin {\mathop{\rm arg\,min}}

\newcommand{\norm}[1]{\left\lVert#1\right\rVert}
\newcommand{\bcc}[1]{\left\{{#1}\right\}}
\newcommand{\brr}[1]{\left({#1}\right)}
\newcommand{\bss}[1]{\left[{#1}\right]}
\newcommand{\ipp}[2]{\left\langle{#1},{#2}\right\rangle}

\makeatletter
\newcommand{\newreptheorem}[2]{\newtheorem*{rep@#1}{\rep@title} 
	\newenvironment{rep#1}[1]{\def\rep@title{#2 \ref*{##1}}\begin{rep@#1}}{\end{rep@#1}}
}
\makeatother

\newreptheorem{lemma}{Lemma}
\newreptheorem{theorem}{Theorem}
\newreptheorem{claim}{Claim}
\newreptheorem{proposition}{Proposition}
\newreptheorem{corollary}{Corollary}

\title{Interactive Teaching Algorithms for Inverse Reinforcement Learning}

\author{
Parameswaran Kamalaruban$^{1}$\thanks{Authors contributed equally to this work.}
\and
Rati Devidze$^{2}$\footnotemark[1]
\and
Volkan Cevher$^{1}$
\And
Adish Singla$^{2}$
\affiliations
$^1$LIONS, EPFL\\
$^2$Max Planck Institute for Software Systems (MPI-SWS)\\
\emails
\{kamalaruban.parameswaran,volkan.cevher\}@epfl.ch,
\{rdevidze,adishs\}@mpi-sws.org
}

\begin{document}
\maketitle

\newtoggle{longversion}
\settoggle{longversion}{true}

\begin{abstract}
We study the problem of inverse reinforcement learning (IRL) with the added twist that the learner is assisted by a helpful teacher. More formally, we tackle the following algorithmic question: How could a teacher provide an informative sequence of demonstrations to an IRL learner to speed up the learning process? We present an interactive teaching framework where a teacher adaptively chooses the next demonstration based on learner's current policy. In particular, we design teaching algorithms for two concrete settings: an omniscient setting where a teacher has full knowledge about the learner's dynamics and a blackbox setting where the teacher has minimal knowledge. Then, we study a sequential variant of the popular MCE-IRL learner and prove convergence guarantees of our teaching algorithm in the omniscient setting. Extensive experiments with a car driving simulator environment show that the learning progress can be speeded up drastically as compared to an uninformative teacher.
\end{abstract}
\section{Introduction}\label{sec:intro}
Imitation Learning, also known as Learning from Demonstrations, enables a learner to acquire new skills by observing a teacher’s behavior. It plays an important role in many real-life learning settings, including human-to-human interaction \cite{buchsbaum2011children,shafto2014rational}, and human-to-robot interaction \cite{schaal1997learning,billard2008robot,argall2009survey,chernova2014robot}. 

Inverse reinforcement learning (IRL) is one of the popular approaches to imitation learning: IRL algorithms operate by first inferring an intermediate reward function explaining the demonstrated behavior, and then obtaining a  policy corresponding to the inferred reward~\cite{russell1998learning,abbeel2004apprenticeship}.
IRL has been extensively studied in the context of designing efficient learning algorithms for a given set of demonstrations \cite{abbeel2004apprenticeship,ratliff2006maximum,ziebart2008maximum,boularias2011relative,wulfmeier2015maximum,finn2016guided}. There has also been some recent work on designing active/interactive IRL algorithms that focus on reducing the number of demonstrations that needs to be requested from a teacher  \cite{kareem2018_repeated,dorsa2017active}. Despite these advances, the problem of generating an optimal sequence of demonstrations to teach an IRL agent is still not well understood.


Motivated by applications of intelligent tutoring systems to teach sequential decision-making tasks, such as surgical training\footnote{Simulators for surgical training: \url{https://www.virtamed.com/en/}\label{footnote:surgical}} or car driving\footnote{Simulators for car driving: \url{https://www.driverinteractive.com/}\label{footnote:driving}},  we study the IRL framework from the viewpoint of a ``teacher" in order to best assist an IRL ``learner". 
\citet{cakmak2012algorithmic,danielbrown2018irl} have studied the problem of teaching an IRL learner in a batch setting, i.e., the teacher has to provide a near-optimal set of demonstrations at once. Their teaching algorithms are  \emph{non-interactive}, i.e., they construct their teaching sequences without incorporating any feedback from the learner and hence unable to adapt the teaching to the learner's progress.

\looseness-1
In real-life pedagogical settings, it is evident that a teacher can leverage the learner's feedback in adaptively choosing next demonstrations/tasks to accelerate the learning progress. For instance, consider a scenario where a driving instructor wants to teach a student certain driving skills. The instructor can easily identify the mistakes/weaknesses of the student (e.g., unable to do rear parking), and then carefully choose tasks that this student should perform, along with specific demonstrations to rectify any mistakes. In this paper, we study the problem of designing an \emph{interactive} teaching algorithm for an IRL learner.





\subsection{Overview of Our Approach}
We consider an interactive teaching framework where at any given time: (i) the teacher observes the learner's current policy, (ii) then, the teacher provides the next teaching task/demonstration, and (iii) the learner performs an update. We design interactive teaching algorithms for two settings:
\begin{itemize}
    \item an ``omniscient" teaching setting where the teacher has full knowledge of the learner's dynamics and can fully observe the learner's current policy.
    \item a ``blackbox" teaching setting where the teacher doesn't know learner's dynamics and has only a noisy estimate of the learner's current policy.
\end{itemize}

In the omniscient teaching setting, we study a sequential variant of the popular IRL algorithm, namely Maximum Causal Entropy (MCE) IRL algorithm \cite{ziebart2008maximum,rhinehart2017first}). Our main idea in designing our omniscient teaching algorithm, \algOmn, is to first reduce the problem of teaching a target policy to that of teaching a corresponding hyperparameter (see Section~\ref{sec.omni.assistive}); then, the teacher greedily steers the learner towards this hyperparameter. We then prove convergence guarantees of the \algOmn algorithm and show that it can significantly reduce the number of demonstrations required to achieve a desired performance of the learner (see Theorem~\ref{greedy-mce-theorem}, Section~\ref{sec.omni.complexity}).  


While omniscient teacher yields strong theoretical guarantees, it's applicability is limited given that the learner's dynamics are unknown and difficult to infer in practical applications. Based on insights from the omniscient teacher, we develop a simple greedy teaching algorithm, \algBbox, for a more practical blackbox setting (see Section~\ref{sec.black.box.teaching}).





We perform extensive experiments in a synthetic learning environment (with both linear and non-linear reward settings) inspired by a car driving simulator~\cite{ng2000algorithms,levine2010feature}. We demonstrate that our teaching algorithms can bring significant improvements in speeding up the learning progress compared to an uninformative teaching strategy that picks demonstrations at random. Furthermore, the performance of the \algBbox algorithm is close to the \algOmn algorithm even though it operates with limited information about the learner. 
\section{Problem Setup}\label{sec.model}
We now formalize the problem addressed in this paper.

\subsection{Environment}
The environment is formally represented by an MDP $\mathcal{M} := \brr{\mathcal{S},\mathcal{A},T,\gamma,P_0,R^E}$. The sets of possible states and actions are denoted by $\mathcal{S}$ and $\mathcal{A}$ respectively. $T: \mathcal{S} \times \mathcal{S} \times \mathcal{A} \rightarrow \bss{0,1}$ captures the state transition dynamics, i.e., $T\brr{s' \mid s,a}$ denotes the probability of landing in state $s'$ by taking action $a$ from state $s$. Here $\gamma$ is the discounting factor, and $P_0: \mathcal{S} \rightarrow \bss{0,1}$ is an initial distribution over states $\mathcal{S}$. We denote a policy $\pi: \mathcal{S} \times \mathcal{A} \rightarrow \bss{0,1}$ as a mapping from a state to a probability distribution over actions. The underlying reward function is given by $R^E: \mathcal{S} \times \mathcal{A} \rightarrow \mathbb{R}$.





\subsection{Interaction between Learner and Teacher}\label{sec.interaction}
 In our setting, we have two entities: a teacher and a sequential IRL learner. 
The teacher has access to the full MDP $\mathcal{M}$ and has a \emph{target policy} $\pi^E$ computed as an optimal policy w.r.t. $R^E$. 
The learner knows the MDP $\mathcal{M}$ but not the reward function $R^E$, i.e., has only access to $\mathcal{M} \textbackslash R^E$. 
 The teacher's goal is to provide an informative sequence of demonstrations to teach the policy $\pi^E$ to the learner. Here, a teacher's demonstration $\xi = \bcc{\brr{s_\tau,a_\tau}}_{\tau = 0,1,\dots}$ is obtained by first choosing an initial state $s_0 \in \mathcal{S}$ (where $P_0(s_0) > 0$) and then choosing a trajectory of state-action pairs obtained by executing the policy $\pi^E$ in the MDP $\mathcal{M}$.

%


We consider an interactive teaching framework with three key steps formally described in Algorithm~\ref{algo:learner-pseudo}. At any time $t$, the teacher observes an estimate of learner's current policy $\pi^L_t$ and accordingly provides an informative demonstration $\xi_t$ to assist the learner. Then, the learner updates it's policy to $\pi^L_{t+1}$ based on demonstration $\xi_t$ and it's internal learning dynamics.






\begin{algorithm}[t!]
	\caption{Interactive Teaching Framework}\label{algo:learner-pseudo}
	\begin{algorithmic}[1]
		\For{$t=1,2,\dots,T$}
	    		    \State teacher observes an estimate of learner's policy  $\pi^L_t$ \;
	    		    \State teacher provides a demonstration $\xi_t$ to the learner
	    		    \State learner updates the policy to $\pi^L_{t+1}$ using $\xi_t$
  		\EndFor
  	\end{algorithmic}
\end{algorithm}

\subsection{Occupancy Measure and Expected Reward}
We introduce the following two notions to formally define the teaching objective. For any policy $\pi$, the occupancy measure $\rho$ and the total expected reward $\nu$ of $\pi$ in the MDP $\mathcal{M}$ are defined as follows respectively:
\begin{align}
\rho^\pi \brr{s,a} :=& \brr{1-\gamma} \pi(a \mid s) \sum_{\tau=0}^\infty \gamma^\tau \mathbb{P}\bcc{S_\tau = s \mid \pi, \mathcal{M}} \label{state-action-visit-eq} \\
\nu^{\pi}:=& \frac{1}{1-\gamma} \sum_{s,a} \rho^\pi \brr{s,a} R^E \brr{s,a} 
\label{state-action-featurecount-eq}
\end{align}
Here, $\mathbb{P}\bcc{S_\tau = s \mid \pi, \mathcal{M}}$ denotes the probability of visiting the state $s$ after $\tau$ steps by following the policy $\pi$.
Similarly, for any demonstration $\xi = \bcc{\brr{s_\tau,a_\tau}}_{\tau = 0,1,\dots}$, we define
\begin{align*}
\rho^\xi \brr{s,a} ~:=~& \brr{1-\gamma} \sum_{\tau = 0}^{\infty} \gamma^\tau \mathbb{I}\bcc{s_\tau = s , a_\tau = a}
\end{align*}
Then for a collection of demonstrations $\Xi = \bcc{\xi_t}_{t=1,2,\dots}$, we have $\rho^{\Xi} \brr{s,a} := \frac{1}{\abs{\Xi}}\sum_{t}{\rho^{\xi_t} \brr{s,a}}$.

\subsection{Teaching Objective}\label{sec.model.objective}
Let $\pi^L$ denote the learner's final policy at the end of teaching. The performance of the policy $\pi^L$ (w.r.t. $\pi^E$) in $\mathcal{M}$ can be evaluated via the following measures (for some fixed $\epsilon > 0$):
\begin{enumerate}
\item $\abs{\nu^{\pi^E} - \nu^{\pi^L}} \leq \epsilon$, ensuring high reward \cite{abbeel2004apprenticeship,ziebart2010modeling}.
\item $D_{\mathrm{TV}}\brr{\rho^{\pi^E},\rho^{\pi^L}} \leq \epsilon$, ensuring that learner's behavior induced by the policy $\pi^L$ matches that of the teacher \cite{ho2016generative}. Here $D_{\mathrm{TV}}\brr{p,q}$ is the total variation distance between two probability measures $p$ and $q$.
\end{enumerate}

IRL learner's goal is to $\epsilon$-\emph{approximate} the teacher's policy w.r.t. one of these performance measures \cite{ziebart2010modeling,ho2016generative}. In this paper, we study this problem from the viewpoint of a teacher in order to provide a near-optimal sequence of demonstrations $\bcc{\xi_t}_{t=1,2,\dots}$ to the learner to achieve the desired goal. The teacher's performance is then measured by the number of demonstrations required to achieve the above objective. 
\section{Omniscient Teaching Setting}\label{sec.omni-teaching}
In this section, we consider the omniscient teaching setting where the teacher knows the learner’s dynamics including the learner’s current parameters at any given time. We begin by introducing a specific learner model that we study for this setting.

\subsection{Learner Model}\label{sec.learner-algos}
We consider a learner implementing an IRL algorithm based on Maximum Causal Entropy approach (MCE-IRL) \cite{ziebart2008maximum,ziebart2010modeling,wulfmeier2015maximum,rhinehart2017first,zhou2018mdce}. First, we discuss the parametric reward function of the learner and then introduce a sequential variant of the MCE-IRL algorithm used in our interactive teaching framework.

\begin{algorithm}[t!]
	\caption{Sequential MCE-IRL}
	\label{algo:sequential-mce-irl}
	\begin{algorithmic}[1]
 		\State \textbf{Initialization:} parameter $\lambda_1$, policy $\pi^L_1 := \pi^L_{\lambda_1}$ 
 		\For{$t=1,2,\dots,T$} 
    			\State receives a demonstration $\xi_t$ with starting state $s_{t,0}$ 
    			\State update  $\lambda_{t+1} \gets \Pi_{\Omega} \bss{\lambda_{t} - \eta_t (\mu^{\pi^L_t , s_{t,0}} - \mu^{\xi_t})}$ \State compute  $\pi^L_{t+1} \gets \textnormal{Soft-Value-Iter}(\mathcal{M} \textbackslash R^E, R^L_{\lambda_{t+1}})$
    		\EndFor  
    	\State \textbf{Output:} policy $\pi^L \gets \pi^L_{T+1}$ 
	\end{algorithmic}
\end{algorithm}

\paragraph{Parametric reward function.}
We consider the learner model with parametric reward function $R^L_\lambda: \mathcal{S} \times \mathcal{A} \rightarrow \mathbb{R}$ where $\lambda \in \mathbb{R}^d$ is a parameter. The reward function  also depends on the learner's feature representation $\phi^L: \mathcal{S} \times \mathcal{A} \rightarrow \mathbb{R}^{d'}$. For linear reward function $R^L_\lambda\brr{s,a} := \ipp{\lambda}{\phi^L\brr{s,a}}$, $\lambda$ represents the weights. 
As an example of non-linear rewards, the function $R^L_\lambda$ could be high-order polynomial in $\lambda$ (see Section~\ref{sec:experiments.nonlinear}). As a more powerful non-linear reward model, $\lambda$ could be the weights of a neural network with $\phi^L \brr{s,a}$ as input layer and $R^L_\lambda\brr{s,a}$ as output \cite{wulfmeier2015maximum}. 



\paragraph{Soft Bellman policies.}
Given a fixed parameter $\lambda$, we model the learner's behavior via the following soft Bellman policy\footnote{For the case of linear reward functions, soft Bellman policies are obtained as the solution to the MCE-IRL optimization problem \cite{ziebart2010modeling,zhou2018mdce}. Here, we extend this idea to model the learner's policy  with the general parametric reward function.  However, we note that for our learner model, it is not always possible to formulate a corresponding optimization problem that leads to the policy form mentioned above.}:
\begin{align}
\pi^L_{\lambda} (a \mid s) ~=~& \exp \brr{Q_\lambda \brr{s, a} - V_\lambda \brr{s}} \label{mce-policy-lambda} \\
V_\lambda \brr{s}  ~=~&\log \sum_{a}{\exp \brr{ Q_\lambda \brr{s, a} }} \nonumber \\
Q_\lambda \brr{s, a} ~=~& R^L_\lambda \brr{s, a} + \gamma \sum_{s'}{T(s' \mid s , a) V_\lambda \brr{s'}}. \nonumber 
\end{align}

For any given $\lambda$, the corresponding policy $\pi^L_\lambda$ can be efficiently computed via Soft-Value-Iteration procedure (see  \cite[Algorithm.~9.1]{ziebart2010modeling}, \cite{zhou2018mdce}).




\paragraph{Sequential MCE-IRL and gradient update.}


We consider a sequential MCE-IRL learner for our interactive setting where, at time step $t$, the learner receives next demonstration $\xi_t = \bcc{(s_{t,\tau} , a_{t,\tau})}_{\tau=0,1,\dots}$ with starting state $s_{t,0}$. Given the learner's current parameter $\lambda_{t}$ and policy $\pi^L_{t} := \pi^L_{\lambda_t}$, 
the learner updates its parameter via an online gradient descent update rule given by $\lambda_{t+1} ~=~ \lambda_t - \eta_t g_t$ with $\eta_t$ as the learning rate.
The gradient $g_t$ at time $t$ is computed as
\[
g_t ~=~ \sum_{s,a} \bss{\rho^{\pi^L_t , s_{t,0}} \brr{s,a} - \rho^{\xi_t} \brr{s,a}} \nabla_\lambda R^L_\lambda \brr{s,a} \big \vert_{\lambda = \lambda_t} ,
\]
where $\rho^{\pi^L_t , s_{t,0}} \brr{s,a}$ is given by Eq.~\eqref{state-action-visit-eq} and computed with $s_{t,0}$ as the only initial state, i.e., $P_0(s_{t,0}) = 1$. 
As shown in Appendix~\ref{sec.appendix.maxent}, $g_t$ can be seen as an empirical counterpart of the gradient of the following loss function:
\begin{equation*}
c\brr{\lambda ; \pi^E} ~=~ \underset{\bcc{\brr{s_\tau , a_\tau}}_\tau \sim \brr{\pi^E , \mathcal{M}}}{\mathbb{E}} \bss{- \sum_\tau \gamma^\tau \log \pi^L_\lambda \brr{a_{\tau} \mid s_{\tau}}}
\end{equation*}
capturing the discounted negative log-likelihood of teacher's demonstrations.
%


For brevity, we define the following quantities:
\begin{align}
\mu^{\pi^L_t , s_{t,0}} ~=~& \sum_{s,a} \rho^{\pi^L_t , s_{t,0}} \brr{s,a} \nabla_\lambda R^L_\lambda \brr{s,a} \big \vert_{\lambda = \lambda_t} \label{eq-feature-expectation} \\
\mu^{\xi_t} ~=~& \sum_{s,a} \rho^{\xi_t} \brr{s,a} \nabla_\lambda R^L_\lambda \brr{s,a} \big \vert_{\lambda = \lambda_t} , \label{eq-feature-expectation-empirical}
\end{align}
and we write the gradient compactly as $g_t = \mu^{\pi^L_t , s_{t,0}} - \mu^{\xi_t}$. 
Algorithm~\ref{algo:sequential-mce-irl} presents the proposed sequential MCE-IRL learner. In particular, we use the online projected gradient descent update given by $\lambda_{t+1} ~\gets~ \Pi_{\Omega} \bss{\lambda_{t} - \eta_t g_t}$,
where $\Omega = \bcc{\lambda: \norm{\lambda}_2 \leq z}$ for large enough $z$ (cf. Section~\ref{sec.omni.complexity.theorem1}). 

\subsection{Omniscient Teacher}\label{sec.omni.assistive}

Next, we present our omniscient teaching algorithm, \algOmn, for  sequential MCE-IRL learner.

\begin{algorithm}[t!]
	\caption{\algOmn for sequential MCE-IRL}\label{algo:sequential-mce-irl-assistive}
	\begin{algorithmic}[1]
		\For{$t=1,2,\dots,T$}
   		 	\State teacher picks $\xi_t$ by solving Eq.~\eqref{mce-omni-teach-eq} 
   		 	\State learner receives $\xi_t$ and updates using Algorithm~\ref{algo:sequential-mce-irl}
		\EndFor  
  	\end{algorithmic}  		
\end{algorithm}

\paragraph{Policy to hyperparameter teaching.}
The main idea in designing our algorithm, \algOmn, is to first reduce the problem of teaching a target policy $\pi^E$ to that of teaching a corresponding hyperparameter, denoted below as $\lambda^*$. Then, under certain technical conditions, teaching this $\lambda^*$ to the learner ensures that the learner's policy $\epsilon$-approximate the target policy $\pi^E$. We defer the technical details to Section~\ref{sec.omni.complexity}.

\paragraph{Omniscient teaching algorithm.}



Now, we design a teaching strategy to greedily steer the learner's parameter $\lambda_t$ towards the target hyperparameter $\lambda^*$. The main idea is to pick a demonstration which minimizes the distance between the learner's current parameter $\lambda_t$ and $\lambda^*$, i.e., minimize $\norm{\lambda_{t+1} - \lambda^*} = \norm{\lambda_{t} - \eta_t {g}_t - \lambda^*}$. We note that, \citet{liu2017iterative,pmlr-v80-liu18b} have used this idea in their iterative machine teaching framework to teach regression and classification tasks. In particular, we consider the following optimization problem for selecting an informative demonstration at time $t$:
\begin{equation}
\label{mce-omni-teach-eq}
\underset{s, \xi}{\mathrm{min}} ~~ \eta_t^2 \norm{\mu^{\pi^L_t , s} - \mu^{\xi}}^2 - 2 \eta_t \ipp{\lambda_t - \lambda^*}{\mu^{\pi^L_t , s} - \mu^{\xi}} , 
\end{equation}
where $s \in \mathcal{S}$ is an initial state (where $P_0(s_0) > 0$) and $\xi$ is a trajectory obtained by executing the policy $\pi^E$ starting from $s$.\footnote{Note that $\xi$ is not constructed synthetically by the teacher, but is obtained by executing policy $\pi^E$ starting from $s$. However, $\xi$ is not just a random rollout of the policy $\pi^E$: When there are multiple possible trajectories from $s$ using $\pi^E$, the teacher can choose the most desirable one as per the joint optimization problem in Eq.~\eqref{mce-omni-teach-eq}.}
The resulting teaching strategy is given in Algorithm~\ref{algo:sequential-mce-irl-assistive}.

\section{Omniscient Teaching Setting: Analysis} \label{sec.omni.complexity}
In this section, we analyze the teaching complexity of our algorithm \algOmn; the proofs are provided Appendix~\ref{sec.appendix.maxent-assistive-teaching}.

\subsection{Convergence to Hyperparameter $\lambda^*$}
In this section, we analyze the teaching complexity, i.e., the total number of time steps $T$ required, to steer the learner towards the target hyperparameter. For this analysis, we quantify the ``richness" of demonstrations (i.e., $\brr{s_{t,0} , \xi_t}$ as solution to Eq.~\eqref{mce-omni-teach-eq}) in terms of providing desired gradients to the learner.
%



For the state  $s_{t,0}$ picked by the teacher at time step $t$, let us first define the following objective function given by $f_t \brr{\mu} := \norm{\lambda_{t} - \lambda^* - \eta_t \mu^{\pi^L_t , s_{t,0}} + \eta_t \mu }^2$.
Note that the optimal solution $\mu^{\mathrm{syn}}_t := \argmin_{\{\mu ~\in~ \R^d\}} f_t \brr{\mu}$ takes the closed form as $\mu^{\mathrm{syn}}_t ~=~ \mu^{\pi^L_t , s_{t,0}} - \frac{1}{\eta_t} \brr{\lambda_t - \lambda^*}$.

Ideally, the teacher would like to provide a demonstration $\xi_t$ (starting from $ s_{t,0}$) for which $\mu^{\xi_t} = \mu^{\mathrm{syn}}_t$. In this case, the learner's hyperparameter converges to $\lambda^*$ after this time step.  However, in our setting, the teacher is restricted to only provide demonstrations obtained by executing the teacher's policy $\pi^E$ in the MDP $\mathcal{M}$. We say that a teacher is $\Delta_{\mathrm{max}}$-rich, if for every $t$, the teacher can provide a demonstration $\xi_t$ that satisfies the following:
\[
\mu^{\xi_t} ~=~ \mu^{\pi^L_t , s_{t,0}} - \beta_t \brr{\lambda_t - \lambda^*} + \delta_t ,
\]
for some $\delta_t$ s.t. $\norm{\delta_t}_{2} \leq \Delta_{\mathrm{max}}$, and $\beta_t \in \bss{0,\frac{1}{\eta_t}}$. The main intuition behind these two quantities is the following: (i) $\beta_t$ bounds the magnitude of the gradient in the desired direction of $\brr{\lambda_t - \lambda^*}$ and (ii) $\delta_t$ accounts for the deviation of the gradient from the desired direction of $\brr{\lambda_t - \lambda^*}$.
The following lemma provides a bound on the teaching complexity of \algOmn to steer the learner's parameter to the target $\lambda^*$.  

\begin{lemma}
\label{greedy-mce-lemma}
Given $\epsilon' > 0$, let the \algOmn be $\Delta_{\mathrm{max}}$-rich with $\Delta_{\mathrm{max}} = \frac{{\epsilon'}^2 \beta^2}{4 \eta_{\mathrm{max}} \bss{4 (1-\beta) z + 1}}$, where $\beta = \min_t \eta_t \beta_t$, and $\eta_{\mathrm{max}} = \max_t \eta_t$. Then for the \algOmn algorithm with $T = \mathcal{O}\brr{\log \frac{1}{\epsilon'}}$, we have $\norm{\lambda_{T} - \lambda^*} \leq \epsilon'$.
\end{lemma}

Note that the above lemma only guarantees the convergence to the target hyperparameter $\lambda^*$. To provide guarantees on our teaching objective, we need additional technical conditions which we discuss below.

\subsection{Convergence to  Policy $\pi^L_{\lambda^*}$} 
We require a smoothness condition on the learner's reward function to ensure that the learner's output policy $\pi^L$ is close to the policy $\pi^L_{\lambda^*}$
in terms of the total expected reward $\nu^{\pi}$ (see Eq.~\eqref{state-action-featurecount-eq}). 
Formally $R^L_\lambda : \mathcal{S} \times \mathcal{A} \rightarrow \mathbb{R}$ is $m$-smooth if the following holds:
\begin{equation}
\max_{s,a} \abs{R^L_\lambda \brr{s,a} - R^L_{\lambda'} \brr{s,a}} ~\leq~ m \norm{\lambda - \lambda'} . 
\label{lips-reward}
\end{equation}
Note that in the linear reward case, the smoothness parameter is $m = \sqrt{d}$.
Given that the reward function $R^L_\lambda$ is smooth, the following lemma illustrates the inherent smoothness of the
soft Bellman policies given in Eq.~\eqref{mce-policy-lambda}. 

\begin{lemma}
\label{smooth-lemma}
Consider an MDP $\mathcal{M}$, and a learner model with parametric reward function $R^L_\lambda$ which is $m$-smooth. Then for any $\lambda,\lambda' \in \mathbb{R}^d$, the associated soft Bellman policies $\pi^L_{\lambda}$ and $\pi^L_{\lambda'}$ satisfy the following: 
\begin{align*}
\abs{\nu^{\pi^L_{\lambda}} - \nu^{\pi^L_{\lambda'}}} ~\leq~& R^E_{\mathrm{max}} \cdot \sqrt{\frac{8 m}{\brr{1 - \gamma}^5} \cdot \norm{\lambda - \lambda'}} \, ,
\end{align*}
where $R^E_{\mathrm{max}} := \max_{s,a}\abs{R^E(s,a)}$.
\end{lemma}
The above lemma suggests that convergence to the target hyperparameter $\lambda^*$ also guarantees  convergence to the policy $\pi^L_{\lambda^*}$ associated with the target hyperparameter in terms of total  expected reward. 

\subsection{Convergence to Policy $\pi^{E}$} \label{sec.omni.complexity.theorem1}


Finally, we need to ensure that the learner's model $(R^L_\lambda, \phi^L)$ is powerful enough to capture the teacher's reward $R^E$ and policy $\pi^E$. Intuitively, this would imply that there exists a target hyperparameter $\lambda^*$ for which the soft Bellmann policy $\pi^L_{\lambda^*}$ has a total expected reward close to that of the teacher's policy $\pi^E$ w.r.t. the reward function $R^E$.
Formally, we say that a learner is $\ell$-learnable (where $\ell > 0$) if there exists a $\lambda \in \Omega$, such that the following holds:
\[
\abs{\nu^{\pi^L_{\lambda}} - \nu^{\pi^E}} ~\leq~ \ell .
\]
In particular, for achieving the desired teaching objective (see Theorem~\ref{greedy-mce-theorem}), we require that the learner is $\frac{\epsilon}{2}$-learnable. This in turn implies that there exists a $\lambda^* \in \Omega$ such that $\abs{\nu^{\pi^L_{\lambda^*}} - \nu^{\pi^E}} ~\leq~ \frac{\epsilon}{2}$.\footnote{\label{lambda-exist-note}For the case of linear rewards, such a $\lambda^*$ is guaranteed to exist and it can be computed efficiently; further details are provided in Appendix~\ref{sec.appendix.policy-hyper-teaching}.}
Then, combining Lemma~\ref{greedy-mce-lemma} and Lemma~\ref{smooth-lemma}, we obtain our main result:
\begin{theorem}
\label{greedy-mce-theorem}
Given $\epsilon > 0$.
Define $R^E_{\mathrm{max}} := \max_{s,a}\abs{R^E(s,a)}$. 
Let $R^L_\lambda$ be $m$-smooth.
Set $\epsilon' = \frac{(1 - \gamma)^5 \epsilon^2}{32 m \brr{R^E_{\mathrm{max}}}^2}$. In addition, we have the following:
\begin{itemize}
    \item Teacher is $\Delta_{\mathrm{max}}$-rich with $\Delta_{\mathrm{max}} = \frac{{\epsilon'}^2 \beta^2}{4 \eta_{\mathrm{max}} \bss{4 (1-\beta) z + 1}}$, where $\beta = \min_t \eta_t \beta_t$, and $\eta_{\mathrm{max}} = \max_t \eta_t$.
    \item Learner is $\frac{\epsilon}{2}$-learnerable.
    \item Teacher has access to $\lambda^*$ such that $\abs{\nu^{\pi^L_{\lambda^*}} - \nu^{\pi^E}} \leq \frac{\epsilon}{2}$. 
\end{itemize}
Then, for the \algOmn algorithm with $T = \mathcal{O}\brr{\log \frac{1}{\epsilon}}$, we have $\abs{\nu^{\pi^{E}} - \nu^{\pi^{L}}} \leq \epsilon$. 
\end{theorem}

The above result states that the number of demonstrations required to achieve the desired objective is only $\mathcal{O}\brr{\log \frac{1}{\epsilon}}$. In practice, this can lead to a drastic speed up in teaching compared to an ``agnostic" teacher (referred to as \algAgn) that provides demonstrations at random, i.e., by randomly choosing the initial state $s_0 \sim P_0$ and then picking a random rollout of the policy $\pi^E$ in the MDP $\mathcal{M}$. In fact, for teaching regression and classification tasks, \citet{liu2017iterative} showed that an omniscient teacher achieves an exponential improvement in teaching complexity as compared to \algAgn teacher who picks the examples at random.

\begin{algorithm}[t!]
	\caption{\algBbox for a sequential IRL learner}\label{algo:black-teacher}
	\begin{algorithmic}[1]
		\State  \textbf{Initialization:} probing parameters ($B$ frequency, $k$ tests) 
		\For{$t=1,2,\dots,T$}
            	\State \textbf{if $t ~\%~ B = 1$:} teacher estimates $\pi^L_t$ using $k$ tests
  	  		\State teacher picks $\xi_t$ by solving Eq.~\eqref{mce-black-teach-alt-true-eq} 
   		 	\State learner receives $\xi_t$ and updates using it's algorithm
  		\EndFor
  	\end{algorithmic}
\end{algorithm}
\section{Blackbox Teaching Setting}\label{sec.black.box.teaching}

In this section, we study a more practical setting where the teacher (i) cannot directly observe the learner's policy $\pi^L_t$ at any time $t$ and
(ii) does not know the learner's dynamics.

\paragraph{Limited observability.}
We first address the challenge of limited observability. The main idea is that in real-world applications, the teacher could approximately infer the learner’s policy $\pi^L_t$ by probing the learner, for instance, by asking the learner to perform certain tasks or ``tests". Here, the notion of a test is to pick an initial state $s \in \mathcal{S}$ (where $P_0(s) > 0$) and then asking the learner to execute the current policy $\pi^L_t$ from $s$. Formally, we characterize this probing via two parameters $(B, k)$: After an interval of $B$ time steps of teaching, the teacher asks learner to perform $k$ ``tests" for every initial state. 


Then, based on observed demonstrations of the learner's policy, the teacher can approximately estimate the occupancy measure $\widehat{\rho^{\pi^L_t , s}} \approx \rho^{\pi^L_t , s}, \forall{s \text{ s.t. } P_0(s) > 0}$. 

\paragraph{Unknown learner's dynamics.}
To additionally deal with the second challenge of unknown learner's dynamics, we propose a simple greedy strategy of picking an informative demonstration. In particular, we pick the demonstration at time $t$ by solving the following optimization problem (see the corresponding equation Eq.~\eqref{mce-omni-teach-eq} for the omniscient teacher): 
\begin{equation}
\label{mce-black-teach-alt-true-eq}
\underset{s, \xi}{\mathrm{max}}  \bigg| \sum_{s',a'} \bcc{\widehat{\rho^{\pi^L_t , s}} \brr{s',a'} - \rho^{\xi} \brr{s',a'}} R^E \brr{s',a'} \bigg| , 
\end{equation}
where $s \in \mathcal{S}$ is an initial state (where $P_0(s_0) > 0$) and $\xi$ is a trajectory obtained by executing the policy $\pi^E$ starting from $s$.
%
Note that we have used the estimate $\widehat{\rho^{\pi^L_t , s}}$ instead of $\rho^{\pi^L_t , s}$. This strategy is inspired from insights of the omniscient teaching setting and can be seen as picking a demonstration $(s, \xi)$ with maximal discrepancy between the learner's current policy and the teacher's policy in terms of  expected reward.

The resulting teaching strategy for the blackbox setting, namely \algBbox, is given in Algorithm~\ref{algo:black-teacher}. 

\section{Experimental Evaluation}\label{sec:experiments}

In this section, we demonstrate the performance of our algorithms in a synthetic learning environment (with both linear and non-linear reward settings) inspired by a car driving simulator~\cite{ng2000algorithms,levine2010feature}.



\paragraph{Environment setup.}
\figref{fig:templates} illustrates a car driving simulator environment consisting of $9$ different lane types (henceforth referred to as tasks), denoted as \texttt{T0}, \texttt{T1}, $\ldots$, \texttt{T8}. Each of these tasks is associated with different driving skills. For instance, task \texttt{T0} corresponds to a basic setup representing a traffic-free highway---it has a very small probability of the presence of another car. However, task  \texttt{T1} represents a crowded highway with $0.25$ probability of encountering a car. Task \texttt{T2} has stones on the right lane, whereas task \texttt{T3} has a mix of both cars and stones.  Similarly, tasks \texttt{T4} has grass on the right lane, and \texttt{T5} has a mix of both grass and cars. Tasks  \texttt{T6}, \texttt{T7}, and \texttt{T8} introduce more complex features such as pedestrians, HOV, and police.

\begin{figure*}[t!]
\centering
\begin{minipage}{0.60\textwidth}
\centering
	\includegraphics[width=1\linewidth]{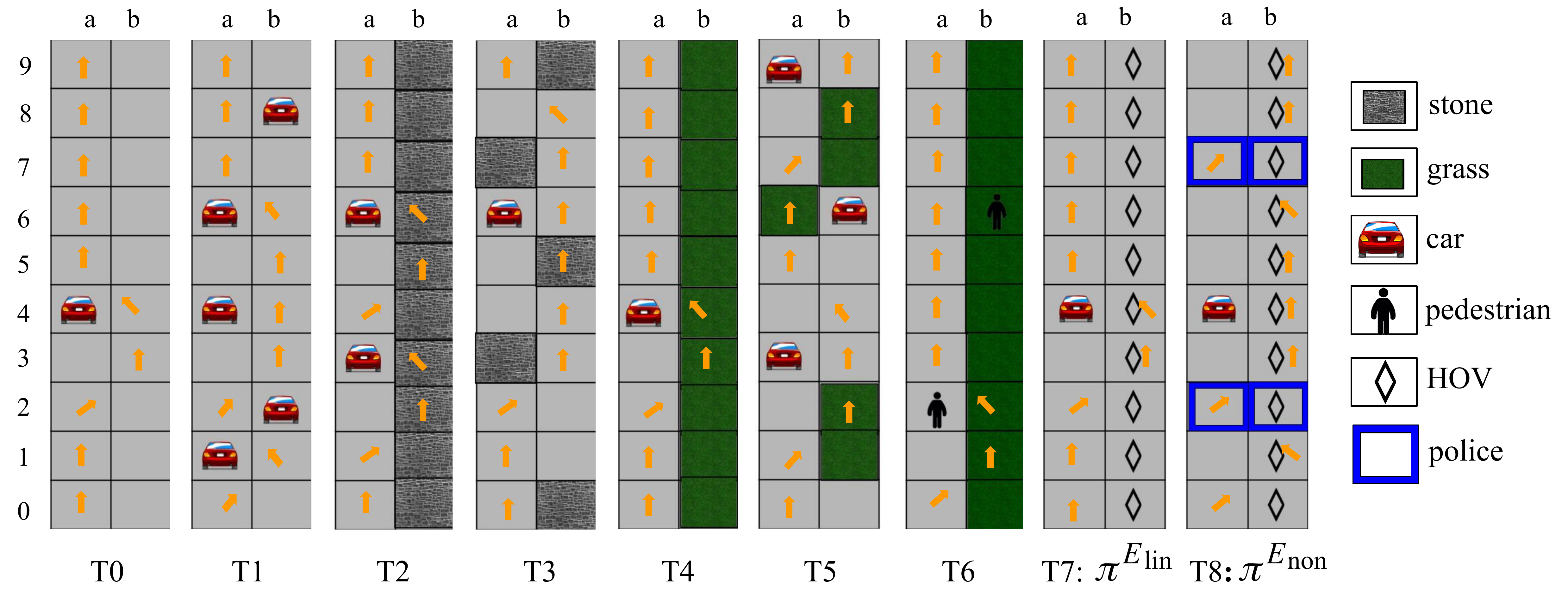}
	\caption{Car environment with $9$ different lane types (tasks). In any given lane, an agent starts from the bottom-left corner and the goal is to reach the top of the lane.  Arrows represent the path taken by the teacher's policy.}
	\label{fig:templates}
\end{minipage}
\qquad
\begin{minipage}{0.28\textwidth}
\centering
\begingroup
\renewcommand{\arraystretch}{1.12} 
	\centering
	\begin{tabular}{|c|c|}
		\hline
	    \textbf{$\phi(s)$} & \textbf{$w$} \\
		\hline
		\texttt{stone} & -1\\
		\texttt{grass} & -0.5\\
		\texttt{car} & -5\\
		\texttt{ped} & -10\\
		\texttt{HOV} & -1\\
		\texttt{police} & 0\\
		\texttt{car-in-f} & -2\\
		\texttt{ped-in-f} & -5\\
		\hline
	\end{tabular}
	\caption{
	$\phi\brr{s}$ represents $8$ features for a state $s$. Weight vector $w$ is used to define the teacher's reward function $R^{E_{\mathrm{lin}}}$.
	}
	\label{tab:wstar}
\endgroup
\end{minipage}
\end{figure*}

\begin{figure*}[t!]
\centering
\begin{minipage}{0.48\textwidth}
    \centering
	\begin{subfigure}{0.47\textwidth}
		\centering
		{
			\includegraphics[height=1\textwidth]{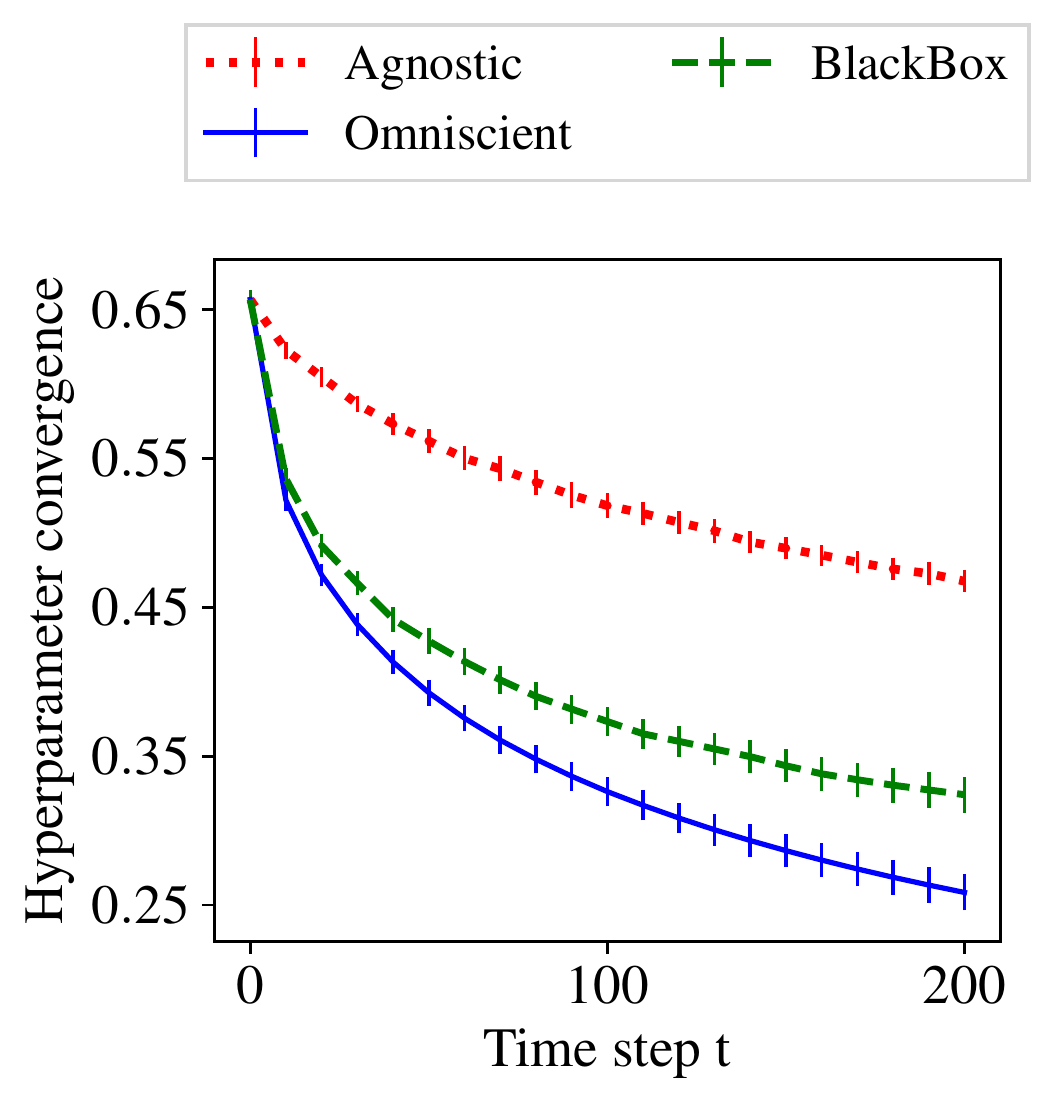}
			\caption{}
			\label{fig:results.lin.lambda}
		}
	\end{subfigure}
	\quad
	\begin{subfigure}{0.47\textwidth}
		\centering
		{
			\includegraphics[height=1\textwidth]{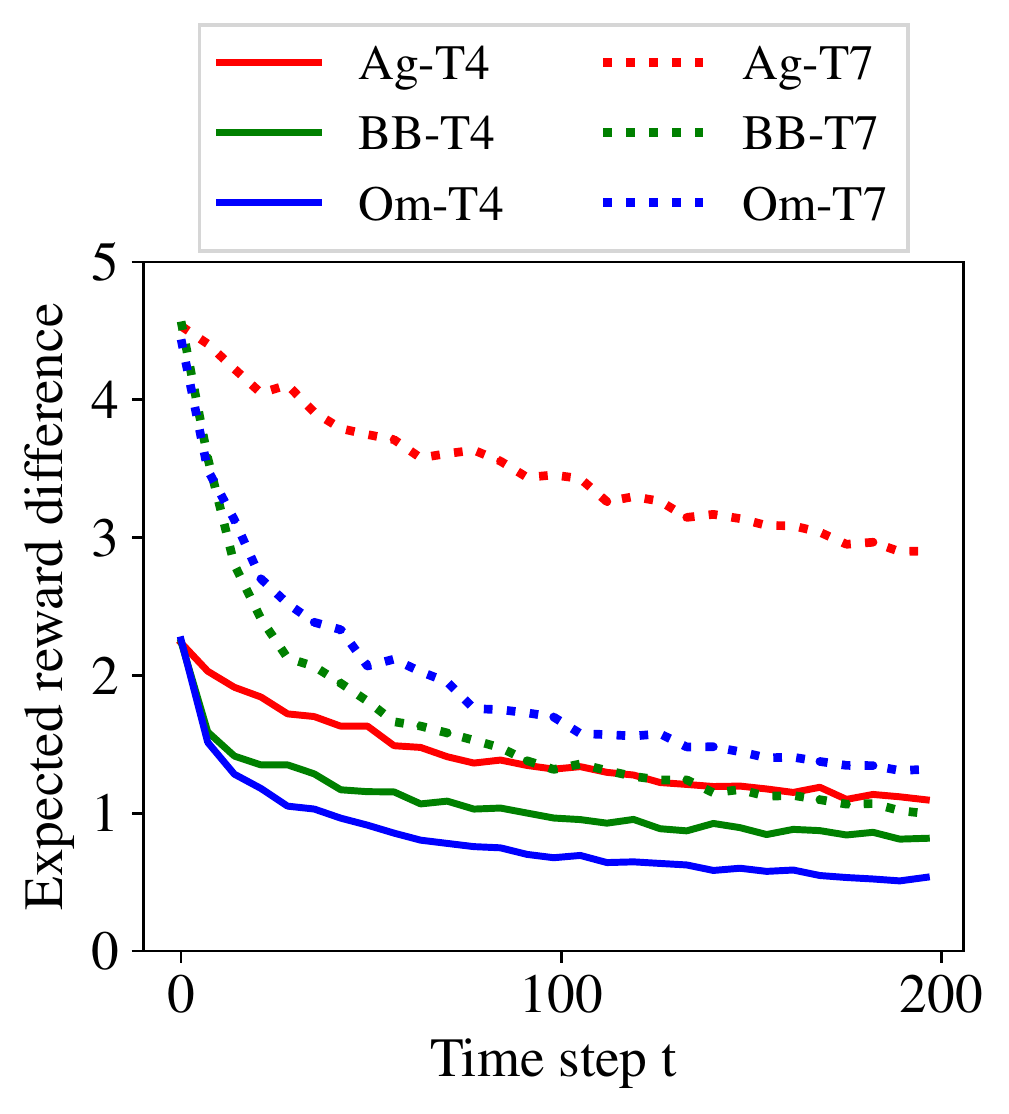}
			\caption{}
			\label{fig:results.lin.reward}
		}
	\end{subfigure}
	\caption{Linear setting of Section~\ref{sec:experiments.linear}. (a) Convergence of the learner's $\lambda_t$ to the target $\lambda^*$. (b) Difference of total expected reward $\nu$ of learner's policy w.r.t. teacher's policy in different tasks.
	}
	\label{fig:results.lin}
\end{minipage}
\quad
\begin{minipage}{0.48\textwidth}
    \centering
	\begin{subfigure}{.47\textwidth}
		\centering
		{
			\includegraphics[height=1\textwidth]{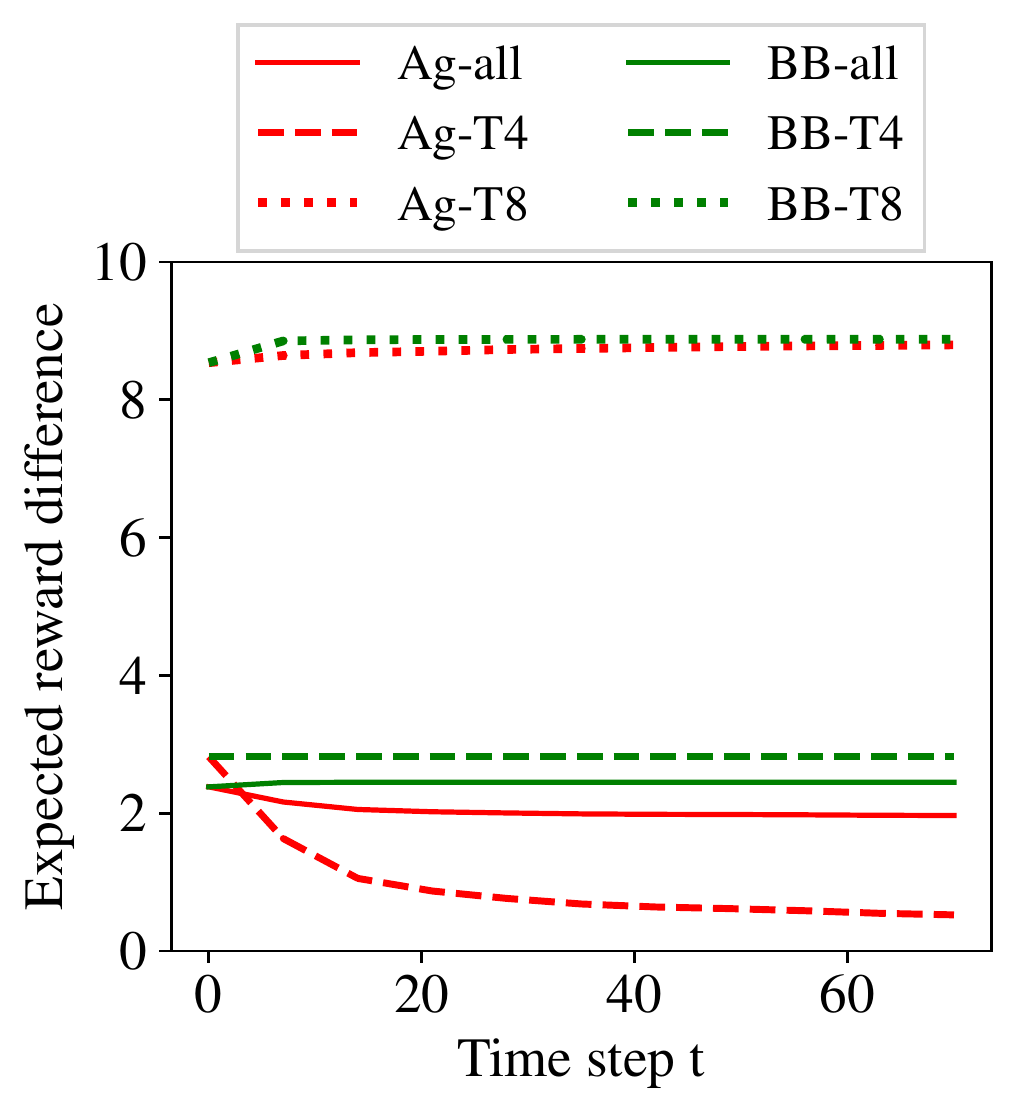}
			\caption{}
			\label{fig:results.non.reward.linl}
		}
	\end{subfigure}
	\quad
	\begin{subfigure}{.47\textwidth}
		\centering
		{
			\includegraphics[height=1\textwidth]{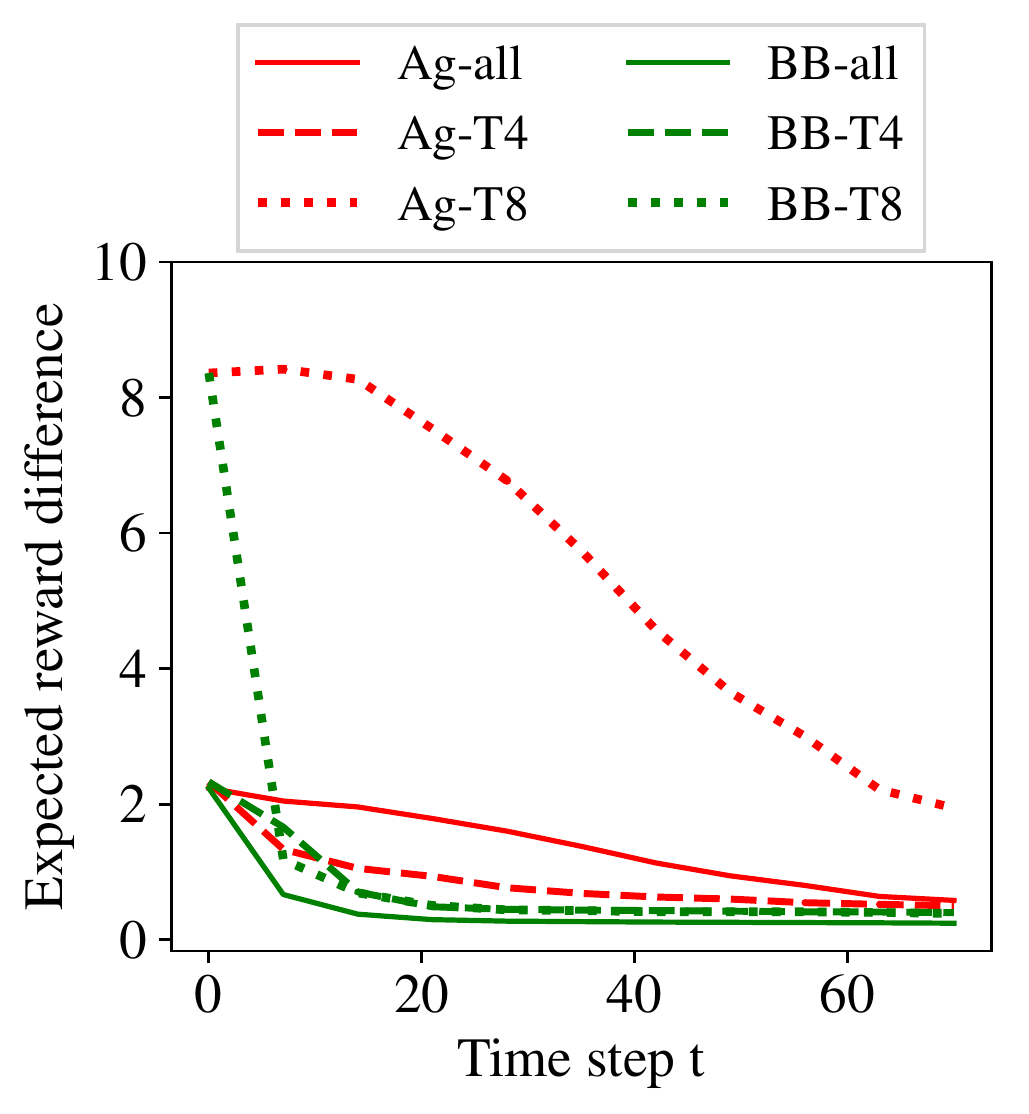}
			\caption{}
			\label{fig:results.non.reward.nolinl}
		}
	\end{subfigure}	
	\caption{
	\looseness-1
	Non-linear setting of Section~\ref{sec:experiments.nonlinear}. (a) Results for a learner model with linear function $R^L_\lambda$ unable to represent the teacher's reward. (b) Results for a learner model with non-linear function $R^L_\lambda$.
	}
	\label{fig:results.non}
\end{minipage}
\end{figure*}

The agent's goal is to navigate starting from an initial state at the bottom left to the top of the lane. The agent can take three different actions given by $\mathcal{A} = $ \{\texttt{left}, \texttt{ straight}, \texttt{ right}\}. 
Action \texttt{left} steers the agent to the left of the current lane. If agent is already in the leftmost lane when taking action \texttt{left}, then the lane is randomly chosen with probability $0.5$. We define similar dynamics for taking action \texttt{right};  action \texttt{straight} means no change in the lane. Irrespective of the action taken, the agent always moves forward. W.l.o.g., we consider that only the agent moves in the environment.

\subsection{Linear Reward Setting}\label{sec:experiments.linear}
First, we study a linear reward setting, and use the notation of $\mathrm{lin}$ in subscript to denote the MDP $\mathcal{M}_{\mathrm{lin}}$, the teacher's reward $R^{E_{\mathrm{lin}}}$, and the teacher's policy $\pi^{E_{\mathrm{lin}}}$.

\paragraph{MDP.} 
We consider lanes corresponding to the first $8$ tasks in the environment, namely \texttt{T0}, \texttt{T1}, $\ldots$, \texttt{T7}.  We have $n$ number of lanes of a given task, each generated randomly according to the tasks' characteristics described above. 
Hence, the total number of states $\abs{\mathcal{S}}$ in our MDP $\mathcal{M}_{\mathrm{lin}}$ is $n \times 8 \times 20$ where each cell represents a state, and each lane is associated with $20$ states (see \figref{fig:templates}). There is one initial state for each lane corresponding to the bottom left cell of that lane.

\paragraph{Teacher's reward and policy.} 
%
Next, we define the reward function $R^{E_{\mathrm{lin}}}$ (i.e., the teacher's reward function). We consider a state-dependent reward that depends on the underlying features of a state $s$ given by the vector $\phi\brr{s}$ as follows:
\begin{itemize}
\item features indicating the type of the current grid cell as \texttt{stone}, \texttt{grass},  \texttt{car}, \texttt{ped}, \texttt{HOV}, and \texttt{police}.
\item features providing some look-ahead information such as whether there is a car or pedestrian in the immediate front cell (denoted as \texttt{car-in-f} and \texttt{ped-in-f}).
\end{itemize}
Given this, we define the teacher's reward function of linear form as $R^{E_{\mathrm{lin}}}\brr{s} = \ipp{w}{\phi\brr{s}}$,
where the $w$ values are given in \figref{tab:wstar}. Teacher's policy $\pi^{E_{\mathrm{lin}}}$ is then computed as the optimal policy w.r.t. this reward function and is illustrated via arrows  in \figref{fig:templates} (\texttt{T0} to \texttt{T7}) representing the path taken by the teacher when driving in this environment.

\paragraph{Learner model.}
We consider the learner model with linear reward function that depends only on state, i.e., $R^L_\lambda\brr{s} = \ipp{\lambda}{\phi^L\brr{s}}$, where $\phi^L\brr{s} = \phi\brr{s}$ as given in \figref{tab:wstar}. The learner is implementing the sequential MCE-IRL in Algorithm~\ref{algo:sequential-mce-irl}, where the learner's prior knowledge is captured by the initial policy $\pi^L_1$ (corresponding to hyperparameter $\lambda_1$). 

In the experiments, we consider the following prior knowledge of the learner: $\pi^L_1$ is initially trained based on demonstrations of $\pi^{E_{\mathrm{lin}}}$ sampled only from the lanes associated with the tasks \texttt{T0}, \texttt{T1}, \texttt{T2}, and \texttt{T3}. Intuitively, the learner initially possesses skills to avoid collisions with cars and to avoid hitting stones while driving. We expect to teach three major skills to the learner, i.e., avoiding grass while driving (task \texttt{T4}, \texttt{T5}, and \texttt{T6}), maintaining distance to pedestrians (task \texttt{T6}), and not to drive on HOV (task \texttt{T7}).

\begin{figure*}[t!]
	\begin{subfigure}[b]{.27\textwidth}
	\centering
	{
		\includegraphics[width=\textwidth]{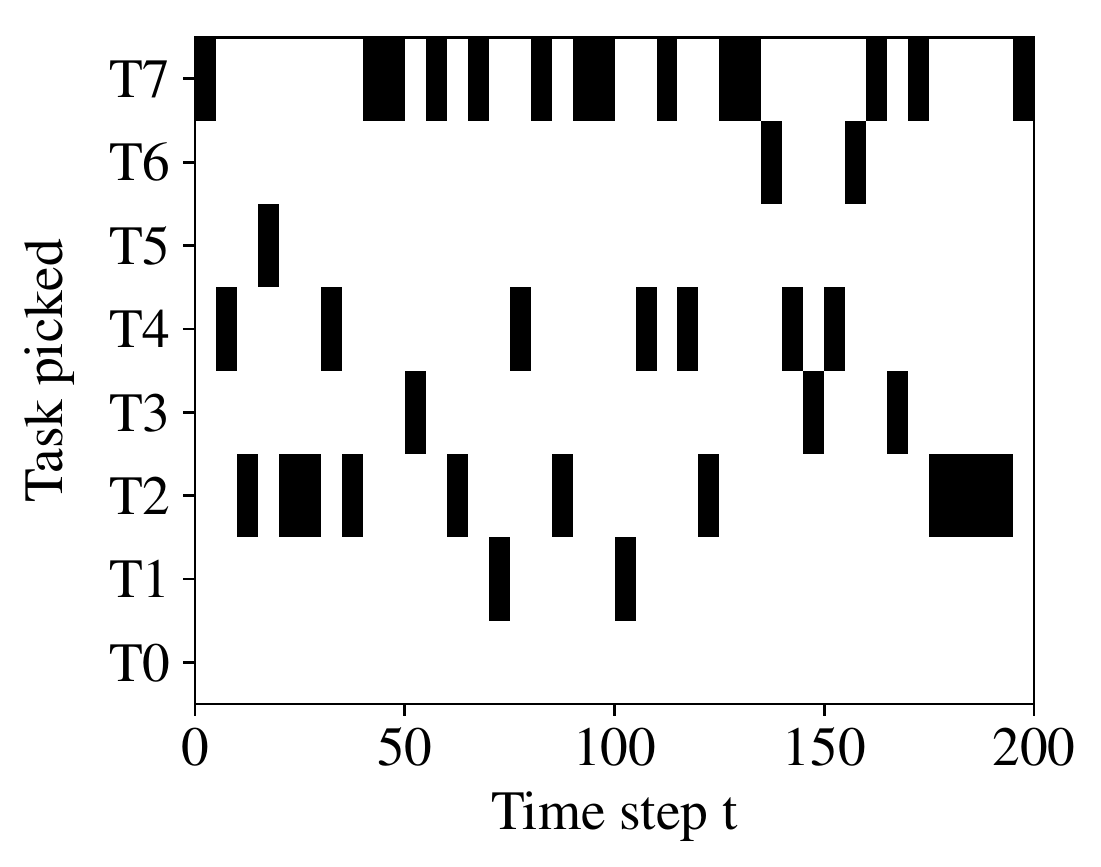}
		\caption{Learner's initial skills: \texttt{T0}}
		\label{fig:curriculum:A}
	}
	\end{subfigure}
	\quad \quad \quad \ \ 
	\begin{subfigure}[b]{.27\textwidth}
		\centering
		{
			\includegraphics[width=\textwidth]{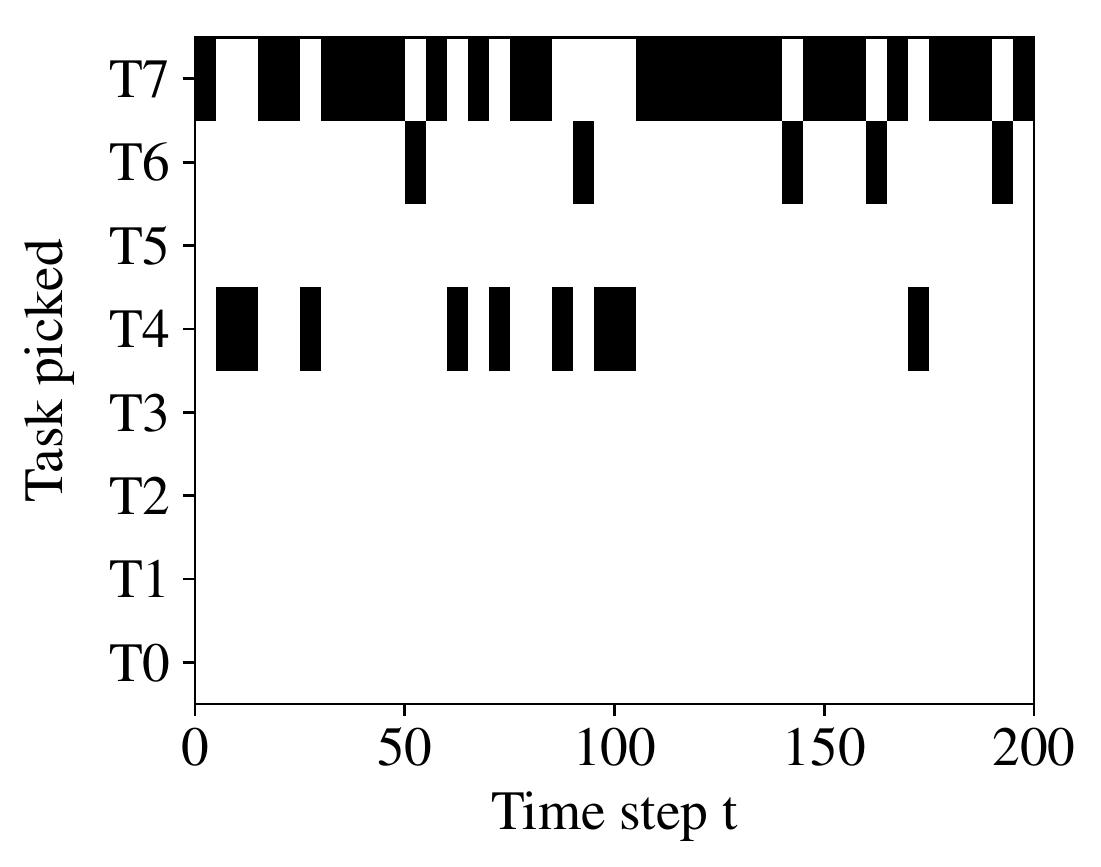}
			\caption{Learner's initial skills: \texttt{T0}--\texttt{T3}}
			\label{fig:curriculum:B}
		}
	\end{subfigure}
    \quad \quad \quad \ \
	\begin{subfigure}[b]{.27\textwidth}
		\centering
		{
			\includegraphics[width=\textwidth]{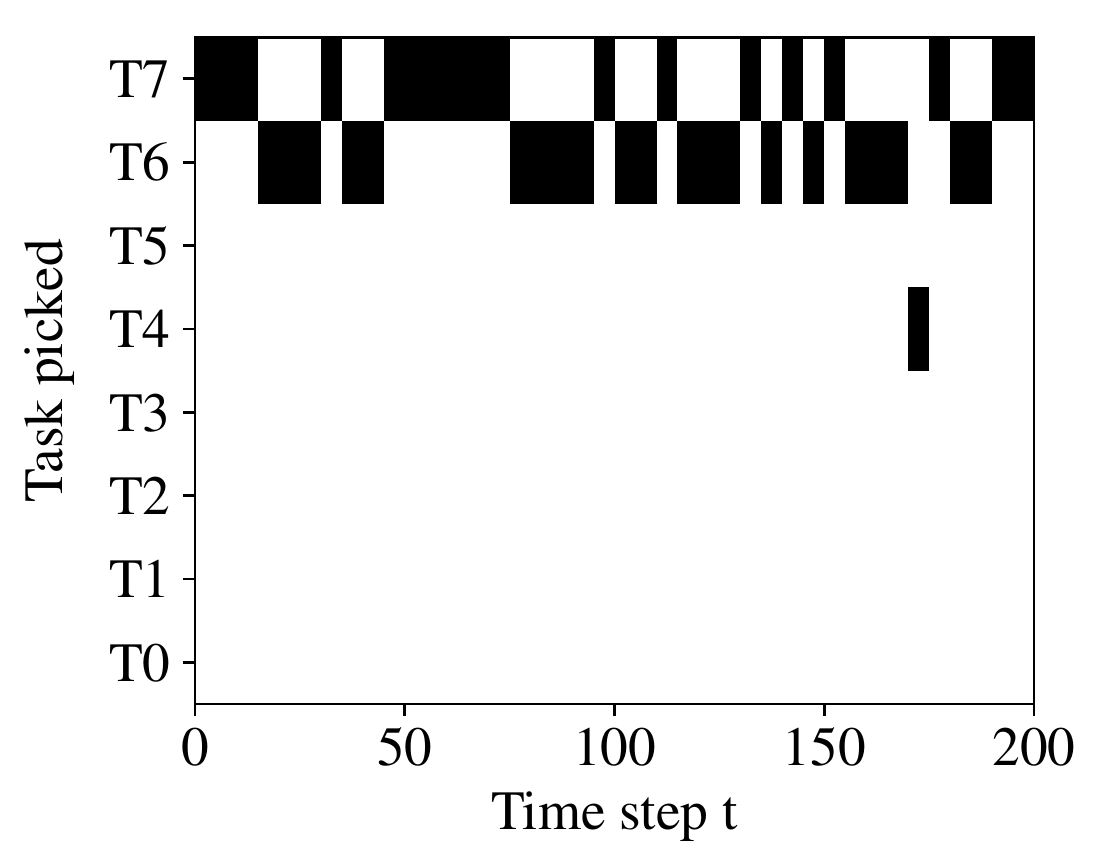}
			\caption{Learner's inital skills: \texttt{T0}--\texttt{T5}}
			\label{fig:curriculum:C}
		}
	\end{subfigure}
   	\quad
	\caption{Teaching curriculum (i.e., the task  associated with the picked state $s_{t,0}$ by \algBbox in Algorithm~\ref{algo:black-teacher}) for three different settings depending on the learner's initial skills trained on (a)  \texttt{T0}, (b)  \texttt{T0}--\texttt{T3}, and (c) \texttt{T0}--\texttt{T5}.}
	\label{fig:curriculum}
\end{figure*}

\paragraph{Experimental results.}
We evaluate the performance of different teaching algorithms, and report the results by averaging over $10$ runs.
We use $n=5$ lanes of each task (i.e., $40$ lanes in total). \algOmn in Algorithm~\ref{algo:sequential-mce-irl-assistive} computes the target hyperparameter $\lambda^*$ as per Footnote~\ref{lambda-exist-note}. For \algBbox in Algorithm~\ref{algo:black-teacher}, we use $B = 5$ and $k=5$.




\figref{fig:results.lin.lambda} shows the convergence of $\norm{\lambda_t - \lambda^*}_2$, where $\lambda^*$ is computed by \algOmn. As expected, \algOmn outperforms other teaching algorithms (\algBbox and \algAgn teacher) that do not have knowledge of the learner's dynamics and are not directly focusing on teaching this hyperparameter $\lambda^*$. Interestingly, the convergence of \algBbox is still significantly faster than \algAgn teacher. In \figref{fig:results.lin.reward}, we consider the teaching objective of total expected reward difference $|\nu^{\pi^{E_{\mathrm{lin}}}} - \nu^{\pi^L_t}|$. We separately plot this objective with starting states limited to task \texttt{T4} and task \texttt{T7}. The results suggest that convergence to target  $\lambda^*$ leads to reduction in expected reward difference (one of the teaching objective). The performance of \algBbox (BB-T7) is even better than \algOmn (Om-T7) for task \texttt{T7}. This is because \algBbox (Eq.~\eqref{mce-black-teach-alt-true-eq}) picks tasks on which the learner's total expected reward difference is highest.

\paragraph{Teaching curriculum.}
In Fig.~\ref{fig:curriculum}, we compare the teaching curriculum of \algBbox for three different settings, where the learner's initial policy $\pi^L_1$ is trained based on demonstrations of $\pi^{E_{\mathrm{lin}}}$ sampled only from the lanes associated with the tasks (a) \texttt{T0}, (b) \texttt{T0}--\texttt{T3}, and (c) \texttt{T0}--\texttt{T5}. The curriculum here refers to the task associated with the state $s_{t,0}$ picked by the teacher at time $t$ to provide the demonstration. In these plots, we can see specific structures and temporal patterns emerging in the curriculum. In particular, we can see that the teaching curriculum focuses on tasks that help the learner acquire new skills. For instance, in Fig.~\ref{fig:curriculum:B}, the teacher primarily picks tasks that provide new skills corresponding to the features \texttt{grass}, \texttt{HOV}, and \texttt{ped}. Recall that, for \algBbox (Algorithm~\ref{algo:black-teacher}), we use $B = 5$ and $k=5$ for our experiments. As a result, the curriculum plots show blocks of length $B = 5$ as the algorithm ends up picking the same task until new tests are performed to get a fresh estimate of the learner's policy.

\subsection{Non-linear Reward Setting}\label{sec:experiments.nonlinear}
Here, we study a non-linear reward setting, and use the notation of $\mathrm{non}$ in subscript for $\mathcal{M}_{\mathrm{non}}$, $R^{E_{\mathrm{non}}}$, and $\pi^{E_{\mathrm{non}}}$.


\paragraph{MDP.} 
We consider the MDP $\mathcal{M}_{\mathrm{non}}$ which consists of lanes corresponding to $5$ tasks: \texttt{T0}, \texttt{T1}, \texttt{T2}, \texttt{T4}, and \texttt{T8} with the total number of states being $n \times 5 \times 20$ as described above.

\paragraph{Teacher's reward and policy.} 
Here, we define the teacher's reward function $R^{E_{\mathrm{non}}}$. The key difference between $R^{E_{\mathrm{non}}}$ and $R^{E_{\mathrm{lin}}}$ is that the teacher in this setting prefers to drive on HOV (gets a $+1$ reward instead of $-1$ reward). However, if police is present while driving on HOV, there is a penalty of $-5$.
Teacher's optimal policy $\pi^{E_{\mathrm{non}}}$ for \texttt{T8} is given in \figref{fig:templates}. For other tasks (apart from \texttt{T7} and \texttt{T8}), the teacher's optimal policy $\pi^{E_{\mathrm{non}}}$ is the same as $\pi^{E_{\mathrm{lin}}}$. 


\paragraph{Learner model.}

We consider two different learner models: (i) with linear reward function $R^L_\lambda\brr{s} = \ipp{\lambda}{\phi^L\brr{s}}$, and (ii) with non-linear reward function $R^L_{\lambda}\brr{s} = \ipp{\lambda_{1:d'}}{\phi^L\brr{s}} + {\ipp{\lambda_{d'+1:2d'}}{\phi^L\brr{s}}}^2$. Here, $\phi^L\brr{s} = \phi\brr{s}$ as in Section~\ref{sec:experiments.linear} and $d'=8$ is the dimension of $\phi\brr{s}$. As a prior knowledge to get $\pi^L_1$ at time $t=1$, we train the learner initially based on demonstrations of $\pi^{E_{\mathrm{non}}}$ sampled only from the lanes associated with the tasks \texttt{T0}, \texttt{T1}, and \texttt{T2}.



\paragraph{Experimental results.}
We use similar experimental settings as in Section~\ref{sec:experiments.linear} (i.e., $n=5$, averaging $10$ runs, etc.). We separately report results for teaching a learner with linear $R^L_{\lambda}$ (see \figref{fig:results.non.reward.linl}) and non-linear $R^L_{\lambda}$ (see \figref{fig:results.non.reward.nolinl}).

\figref{fig:results.non.reward.linl} shows that both \algBbox and \algAgn teacher are unable to make progress in teaching task \texttt{T8} to the learner.
Interestingly, the overall performance measuring the total expected reward difference on the whole MDP for \algBbox (BB-all) is worse compared to \algAgn (Ag-all): This is an artifact of \algBbox's strategy in Eq.~\eqref{mce-black-teach-alt-true-eq} being stuck in always picking task \texttt{T8} (even though the learner is not making any progress).
\figref{fig:results.non.reward.nolinl} illustrates that the above mentioned limitations are gone when teaching a learner using a non-linear reward function. Here, the rate of reduction in total expected reward difference is significantly faster for \algBbox as compared to \algAgn, as was also observed in \figref{fig:results.lin.reward}. 

These results demonstrate that the learning progress can be sped up significantly by an adaptive teacher even with limited knowledge about the learner, as compared to an uninformative teacher. These results also signify the importance that the learner's representation of feature space and reward function should be powerful enough to learn the desired behavior.

\section{Related Work}\label{sec:relatedwork}


\paragraph{Imitation learning.} The two popular approaches for imitation learning include (i) behavioral cloning, which directly replicates the desired behavior \cite{bain1999framework}, and (ii) inverse reinforcement learning (IRL) which infers the reward function explaining the desired behavior \cite{russell1998learning}. We refer the reader to a recent survey by \citet{osa2018algorithmic} on imitation learning. 
 
\citet{kareem2018_repeated} have studied interactive IRL algorithms that actively request the teacher to provide suitable demonstrations with the goal of reducing the number of interactions. However, the key difference in our approach is that we take the viewpoint of a teacher in how to best assist the learning agent by providing an optimal sequence of demonstrations. Our approach is inspired by real-life pedagogical settings where carefully choosing the teaching demonstrations and tasks can accelerate the learning progress \cite{ho2016showing}. \citet{hadfield2016cooperative} have studied the value alignment problem in a game-theoretic setup, and provided an approximate scheme to generate instructive demonstrations for an IRL agent. In our work, we devise a systematic procedure (with convergence guarantees for the omniscient setting) to choose an optimal sequence of demonstrations, by taking into account the learner's dynamics. 

\paragraph{Steering and teaching in reinforcement learning.}
A somewhat different but related problem setting is that of \emph{reward shaping} and \textit{environment design} where the goal is to modify/design the reward function to guide/steer the behavior of a learning agent \cite{ng1999policy,zhang2009policy,sorg2010online_gradient_asc}. Another related problem setting is considered in the \textit{advice-based interaction} framework (e.g., \cite{torrey2013teaching,DBLP:conf/ijcai/AmirKKG16}),  where the goal is to communicate advice to a suboptimal agent on how to act in the world.
\paragraph{Algorithmic teaching.}
Another line of research relevant to our work is that of algorithmic teaching.  Here, one studies the interaction between a teacher and a learner where the teacher's objective is to find an optimal training sequence to steer the learner towards a desired goal \cite{goldman1995complexity,liu2017iterative,DBLP:journals/corr/ZhuSingla18}. Algorithmic teaching provides a rigorous formalism for a number of real-world applications such as personalized education and intelligent tutoring systems \cite{patil2014optimal,rafferty2016faster,hunziker2018teachingmultiple}, social robotics \cite{cakmak2014eliciting}, and human-in-the-loop systems \cite{singla2014near,singla2013actively}. Most of the work in machine teaching is in a batch setting where the teacher provides a batch of teaching examples at once without any adaptation.  The question of how a teacher should adaptively select teaching examples for a learner has been addressed recently but only in the supervised learning setting~\cite{melo2018interactive,pmlr-v80-liu18b,chen2018understanding,yeo2019classroom}.


\paragraph{Teaching sequential tasks.}
\citet{DBLP:conf/uai/WalshG12,cakmak2012algorithmic} have studied algorithmic teaching for sequential decision-making tasks. \citet{cakmak2012algorithmic} studied the problem of teaching an IRL agent in the batch setting, i.e., the teacher has to provide a near-optimal set of demonstrations at once. They considered the IRL algorithm by \citet{ng2000algorithms}, which could only result in inferring an equivalent class of reward weight parameters for which the observed behavior is optimal.  In recent work, \citet{danielbrown2018irl} have extended the previous work of \citet{cakmak2012algorithmic} by showing that the teaching problem can be formulated as a set cover problem.  However, their teaching strategy does not take into account how the learner progresses (i.e., it is non-interactive). In contrast, we study interactive teaching setting to teach a sequential MCE-IRL  algorithm \cite{ziebart2008maximum,rhinehart2017first}. This interactive setting, in turn, allows the teacher to design a personalized and adaptive curriculum important for efficient learning \cite{tadepalli2008learning}. \citet{haug_teaching_2018} have studied the problem of teaching an IRL agent adaptively; however, they consider a very different setting where the teacher and the learner have a mismatch in their worldviews.
\section{Conclusions}\label{sec:conclusions}
We studied the problem of designing interactive teaching algorithms to provide an informative sequence of demonstrations to a sequential IRL learner. In an omniscient teaching setting, we presented \algOmn which achieves the teaching objective with $\mathcal{O}\brr{\log \frac{1}{\epsilon}}$ demonstrations. Then, utilizing the insights from \algOmn, we proposed \algBbox for a more practical blackbox setting. We showed the effectiveness of our algorithms via extensive experiments in a learning environment inspired by a car driving simulator.

As future work, we will investigate extensions of our ideas to more complex environments; we hope that, ultimately, such extensions will provide a basis for designing teaching strategies for intelligent tutoring systems (see Footnote~\ref{footnote:surgical} and Footnote~\ref{footnote:driving}). It would also be interesting to benchmark active imitation methods for MCE-IRL learner using our omniscient teacher (see, e.g., \cite{danielbrown2018irl}). Our results are also important in getting a better understanding of the robustness of MCE-IRL learner against adversarial training-set poisoning attacks. Our fast convergence results in Theorem~\ref{greedy-mce-theorem} suggests that the MCE-IRL learner is actually brittle to adversarial attacks, and designing a robust MCE-IRL is an important direction of future work.

\vspace{2mm}
{\bfseries Acknowledgments.} This work was supported in part by the Swiss National Science Foundation (SNSF) under grant number 407540\_167319.


\bibliography{main}
\bibliographystyle{named}
\iftoggle{longversion}{
\clearpage
\onecolumn
\appendix
{\allowdisplaybreaks

\section{Gradient for Sequential MCE-IRL}\label{sec.appendix.maxent}


\noindent
In this section, we show that $g_t$ can be seen as an empirical counterpart of the gradient of the following loss function:
\begin{equation*}
c\brr{\lambda ; \pi^E} ~=~ \underset{\bcc{\brr{s_\tau , a_\tau}}_\tau \sim \brr{\pi^E , \mathcal{M}}}{\mathbb{E}} \bss{- \sum_\tau \gamma^\tau \log \pi^L_\lambda \brr{a_{\tau} \mid s_{\tau}}}
\end{equation*}
capturing the discounted negative log-likelihood of teacher's demonstrations. In Proposition~\ref{mce-gradient-proposition}, we show that gradient $\nabla_\lambda c\brr{\lambda ; \pi^E}$ is given by $\sum_{s,a} \bss{\rho^{\pi^L_\lambda} \brr{s,a} - \rho^{\pi^E} \brr{s,a}} \nabla_\lambda R^L_\lambda \brr{s,a}$ (see Eq.~\eqref{appendix.eq:max-ent-grad-exp}).

\vspace{2mm}
\noindent
Given the teacher's demonstration $\xi_t$ with starting state $s_{t,0}$, we compute $g_t$ as follows:
\begin{itemize}
    \item In Eq.~\eqref{appendix.eq:max-ent-grad-exp}, consider the gradient component corresponding to teacher's policy, i.e., $- \sum_{s,a} \rho^{\pi^E} \brr{s,a} \nabla_\lambda R^L_\lambda \brr{s,a}$. We replace $\rho^{\pi^E} \brr{s,a}$ with it's empirical counterpart using $\xi_t$. This results in the following component in $g_t$:
    \[
    - \sum_{s,a} \rho^{\xi_t} \brr{s,a} \nabla_\lambda R^L_\lambda \brr{s,a} \big \vert_{\lambda = \lambda_t}
    \]
    \item In Eq.~\eqref{appendix.eq:max-ent-grad-exp}, consider the gradient component corresponding to learner's policy, i.e., $\sum_{s,a} \rho^{\pi^L_\lambda} \brr{s,a} \nabla_\lambda R^L_\lambda \brr{s,a}$. We compute $\rho^{\pi^L_\lambda} \brr{s,a}$ with $s_{t,0}$ as the only initial state, i.e., $P_0(s_{t,0}) = 1$. This results in the following component in $g_t$:
    \[
    \sum_{s,a} \rho^{\pi^L_t , s_{t,0}} \brr{s,a} \nabla_\lambda R^L_\lambda \brr{s,a} \big \vert_{\lambda = \lambda_t}
    \]
\end{itemize}
Hence, $g_t$ given by 
\[
g_t ~=~ \sum_{s,a} \bss{\rho^{\pi^L_t , s_{t,0}} \brr{s,a} - \rho^{\xi_t} \brr{s,a}} \nabla_\lambda R^L_\lambda \brr{s,a} \big \vert_{\lambda = \lambda_t} .
\]

\begin{proposition}
\label{mce-gradient-proposition}
Consider the loss function defined as follows:
\[
c\brr{\lambda ; \pi^E} ~:=~ \underset{\bcc{\brr{s_\tau , a_\tau}}_\tau \sim \brr{\pi^E , \mathcal{M}}}{\mathbb{E}} \bss{- \sum_\tau \gamma^\tau \log \pi^L_\lambda \brr{a_{\tau} \mid s_{\tau}}} .
\]
Then the gradient of $c\brr{\lambda ; \pi^E}$ w.r.t. $\lambda$ is given by 
\begin{equation}
\label{appendix.eq:max-ent-grad-exp}
\nabla_\lambda c\brr{\lambda ; \pi^E} ~=~ \sum_{s,a} \bss{\rho^{\pi^L_\lambda} \brr{s,a} - \rho^{\pi^E} \brr{s,a}} \nabla_\lambda R^L_\lambda \brr{s,a} .
\end{equation}
\end{proposition}

\begin{proof}
We will make use of the following quantities as part of the proof:
\begin{itemize}
    \item $P^L_{\tau}\brr{s}$ is defined as the probability of visiting state $s$ at time $\tau$ by the policy $\pi^L_\lambda$.
    \item $P^L_{\tau}\brr{s, a}$ is defined as the probability of taking action $a$ from state $s$ at time $\tau$ by the policy $\pi^L_\lambda$.
    \item $P^E_{\tau}\brr{s}$ is defined as the probability of visiting state $s$ at time $\tau$ by the policy $\pi^E$.
    \item $P^E_{\tau}\brr{s, a}$ is defined as the probability of taking action $a$ from state $s$ at time $\tau$ by the policy $\pi^E$.
\end{itemize}

First, we rewrite $c\brr{\lambda ; \pi^E}$ as follows:
\begin{align*}
c\brr{\lambda ; \pi^E} ~=~& - \sum_{\tau = 0}^{\infty} \sum_{s , a} P^E_\tau\brr{s} P^E_\tau\brr{a \mid s} \gamma^\tau \bss{Q_\lambda \brr{s, a} - V_\lambda \brr{s}} \\
~=~& - \sum_{s,a} P_0\brr{s} P^E_0\brr{a \mid s} \bss{Q_\lambda \brr{s,a} - V_\lambda \brr{s}} \\
& - \sum_{s,a} P^E_1\brr{s} P^E_1\brr{a \mid s} \gamma \bss{Q_\lambda \brr{s,a} - V_\lambda \brr{s}} \\ 
& - \sum_{s,a} P^E_2\brr{s} P^E_2\brr{a \mid s} \gamma^2 \bss{Q_\lambda \brr{s,a} - V_\lambda \brr{s}} \\
& - \sum_{s,a} P^E_3\brr{s} P^E_3\brr{a \mid s} \gamma^3 \bss{Q_\lambda \brr{s,a} - V_\lambda \brr{s}} \\
& \dots \\
~=~& \sum_{s,a} P_0\brr{s} P^E_0\brr{a \mid s} V_\lambda \brr{s} \\
& - \sum_{s,a} P_0\brr{s} P^E_0\brr{a \mid s} Q_\lambda \brr{s,a} + \gamma \sum_{s,a} P^E_1\brr{s} P^E_1\brr{a \mid s} V_\lambda \brr{s} \\
& - \gamma \sum_{s,a} P^E_1\brr{s} P^E_1\brr{a \mid s} Q_\lambda \brr{s,a} + \gamma^2 \sum_{s,a} P^E_2\brr{s} P^E_2\brr{a \mid s} V_\lambda \brr{s} \\
& - \gamma^2 \sum_{s,a} P^E_2\brr{s} P^E_2\brr{a \mid s} Q_\lambda \brr{s,a} + \gamma^3 \sum_{s,a} P^E_3\brr{s} P^E_3\brr{a \mid s} V_\lambda \brr{s}\\
&\dots \\
~=~& \textcolor{red}{\sum_{s} P_0\brr{s} V_\lambda \brr{s}}  \\
& \textcolor{blue}{- \sum_{s,a} P^E_0\brr{s,a} Q_\lambda \brr{s,a} + \gamma \sum_{s} P^E_1\brr{s} V_\lambda \brr{s}} \\
& \textcolor{magenta}{- \gamma \sum_{s,a} P^E_1\brr{s,a} Q_\lambda \brr{s,a} + \gamma^2 \sum_{s} P^E_2\brr{s} V_\lambda \brr{s}} \\
& \textcolor{violet}{- \gamma^2 \sum_{s,a} P^E_2\brr{s,a} Q_\lambda \brr{s,a} + \gamma^3 \sum_{s} P^E_3\brr{s} V_\lambda \brr{s}}\\
& \dots \\
\end{align*}
Now, we compute the gradient of the \textcolor{red}{first term}: 
\begin{align*}
\nabla_\lambda \sum_{s} P_0\brr{s} V_\lambda \brr{s} ~=~& \sum_{s} P_0\brr{s} \nabla_\lambda V_\lambda  \brr{s} \\
~=~& \sum_{s} P_0\brr{s} \nabla_\lambda \log \sum_{a}{\exp \brr{ Q_\lambda \brr{s, a} }} \\
~=~& \sum_{s} P_0\brr{s} \frac{\sum_{a}{\exp \brr{ Q_\lambda \brr{s, a} } \nabla_\lambda Q_\lambda \brr{s, a}}}{\sum_{a}{\exp \brr{ Q_\lambda \brr{s, a} }}} \\
~=~& \sum_{s} P_0\brr{s} \sum_{a} \pi^L_{\lambda} (a \mid s) \nabla_\lambda Q_\lambda \brr{s, a} \\
~=~& \sum_{s , a} P^L_0\brr{s , a} \bss{\nabla_\lambda R^L_\lambda \brr{s, a} + \gamma \sum_{s'}{T(s' \mid s , a) \nabla_\lambda V_\lambda \brr{s'}}} \\
~=~& \sum_{s , a} P^L_0\brr{s , a} \nabla_\lambda R^L_\lambda \brr{s, a} + \sum_{s'} \sum_{s , a} P^L_0\brr{s , a} T(s' \mid s , a) \gamma \nabla_\lambda V_\lambda \brr{s'} \\
~=~& \sum_{s , a} P^L_0\brr{s , a} \nabla_\lambda R^L_\lambda \brr{s, a} + \sum_{s'} P^L_1\brr{s'} \gamma \nabla_\lambda V_\lambda \brr{s'} \\
~=~& \sum_{s , a} P^L_0\brr{s , a} \nabla_\lambda R^L_\lambda \brr{s, a} + \sum_{s , a} P^L_1\brr{s , a} \gamma \nabla_\lambda R^L_\lambda \brr{s, a} + \sum_{s'} P^L_2\brr{s'} \gamma^2 \nabla_\lambda V_\lambda \brr{s'} \\
& \vdots \\
~=~& \sum_{s , a} \rho^{\pi^L_\lambda} \brr{s,a} \nabla_\lambda R^L_\lambda \brr{s , a} . 
\end{align*}
Now, we compute the gradient of the \textcolor{blue}{second term}: 
\begin{align*}
& - \sum_{s,a} P^E_0\brr{s,a} \nabla_\lambda Q_\lambda \brr{s,a} + \gamma \sum_{s} P^E_1\brr{s} \nabla_\lambda V_\lambda \brr{s} \\
~=~& - \sum_{s,a} P^E_0\brr{s,a} \bss{\nabla_\lambda R^L_\lambda \brr{s,a} + \gamma \sum_{s'}{T(s' \mid s,a) \nabla_\lambda V_\lambda \brr{s'}}} + \gamma \sum_{s} P^E_1\brr{s} \nabla_\lambda V_\lambda \brr{s} \\
~=~& - \sum_{s,a} P^E_0\brr{s,a} \nabla_\lambda R^L_\lambda \brr{s,a} - \gamma \sum_{s'} \sum_{s,a} P^E_0\brr{s,a} T(s' \mid s,a) \nabla_\lambda V_\lambda \brr{s'} + \gamma \sum_{s} P^E_1\brr{s} \nabla_\lambda V_\lambda \brr{s} \\
~=~& - \sum_{s,a} P^E_0\brr{s,a} \nabla_\lambda R^L_\lambda \brr{s,a} - \gamma \sum_{s'} P^E_1\brr{s'} \nabla_\lambda V_\lambda \brr{s'} + \gamma \sum_{s} P^E_1\brr{s} \nabla_\lambda V_\lambda \brr{s} \\
~=~& - \sum_{s,a} P^E_0\brr{s,a} \nabla_\lambda R^L_\lambda \brr{s,a} .
\end{align*}
By following similar steps, the gradient of all the terms except the \textcolor{red}{first term} is given by
\[
- \sum_{\tau = 0}^{\infty} \sum_{s,a} P^E_\tau\brr{s,a} \gamma^\tau \nabla_\lambda R^L_\lambda \brr{s,a} ~=~ - \sum_{s , a} \rho^{\pi^E} \brr{s,a} \nabla_\lambda R^L_\lambda \brr{s , a} . 
\]
Thus the full gradient is given by
\[
\nabla_\lambda c\brr{\lambda ; \pi^E} ~=~ \sum_{s,a} \bss{\rho^{\pi^L_\lambda} \brr{s,a} - \rho^{\pi^E} \brr{s,a}} \nabla_\lambda R^L_\lambda \brr{s,a} .
\]
\end{proof}

\section{Proofs}\label{sec.appendix.maxent-assistive-teaching}

In this section, we provide proofs of Lemma~\ref{greedy-mce-lemma} and Lemma~\ref{smooth-lemma}.


\subsection{Proof of Lemma~\ref{greedy-mce-lemma}}
\begin{proof}
The demonstration $\xi_t$ (starting from $ s_{t,0}$) provided by the teacher at time $t$ satisfies the following:
\begin{equation}
\label{proof-lemma-1-eq-1}
\mu^{\xi_t} ~=~ \mu^{\pi^L_t , s_{t,0}} - \beta_t \brr{\lambda_t - \lambda^*} + \delta_t ,
\end{equation}
for some $\delta_t$ s.t. $\norm{\delta_t}_{2} \leq \Delta_{\mathrm{max}}$, and $\beta_t \in \bss{0,\frac{1}{\eta_t}}$. Consider: 
\begin{align*}
\norm{\lambda_{t+1} - \lambda^*}^2 ~=~& \norm{\Pi_{\Omega} \bss{\lambda_{t} - \eta_t (\mu^{\pi^L_t , s_{t,0}} - \mu^{\xi_t})} - \lambda^*}^2 \\
~\overset{(i)}{\leq}~& \norm{\lambda_t - \eta_t \brr{\mu^{\pi^L_t , s_{t,0}} - \mu^{\xi_t}} - \lambda^*}^2 \\
~=~& \norm{\lambda_t - \lambda^*}^2 - 2 \eta_t \ipp{\lambda_t - \lambda^*}{\mu^{\pi^L_t , s_{t,0}} - \mu^{\xi_t}} + \eta_t^2 \norm{\mu^{\pi^L_t , s_{t,0}} - \mu^{\xi_t}}^2 \\
~\overset{(ii)}{=}~& \norm{\lambda_t - \lambda^*}^2 - 2 \eta_t \ipp{\lambda_t - \lambda^*}{\beta_t \brr{\lambda_t - \lambda^*} - \delta_t} + \eta_t^2 \norm{\beta_t \brr{\lambda_t - \lambda^*} - \delta_t}^2 \\
~=~& \norm{\lambda_t - \lambda^*}^2 - 2 \eta_t \beta_t \ipp{\lambda_t - \lambda^*}{\lambda_t - \lambda^*} + \eta_t^2 \beta_t^2 \norm{\lambda_t - \lambda^*}^2 \\
& + 2 \eta_t \ipp{\lambda_t - \lambda^*}{\delta_t} + \eta_t^2 \norm{\delta_t}^2 - 2 \eta_t^2 \beta_t \ipp{\lambda_t - \lambda^*}{\delta_t} \\
~=~& \brr{1 - \eta_t \beta_t}^2 \norm{\lambda_t - \lambda^*}^2 + 2 \eta_t (1-\eta_t \beta_t) \ipp{\lambda_t - \lambda^*}{\delta_t} + \eta_t^2 \norm{\delta_t}^2 \\
~\overset{(iii)}{\leq}~& \brr{1 - \eta_t \beta_t}^2 \norm{\lambda_t - \lambda^*}^2 + 2 \eta_t (1-\eta_t \beta_t) \norm{\lambda_t - \lambda^*}\norm{\delta_t} + \eta_t^2 \norm{\delta_t}^2 \\
~\overset{(iv)}{\leq}~& \brr{1 - \beta}^2 \norm{\lambda_t - \lambda^*}^2 + 2 \eta_t (1-\beta) \norm{\lambda_t - \lambda^*}\norm{\delta_t} + \eta_t \norm{\delta_t} \\
~\overset{(v)}{\leq}~& \brr{1 - \beta}^2 \norm{\lambda_t - \lambda^*}^2 + 4 \eta_t (1-\beta) z \norm{\delta_t} + \eta_t \norm{\delta_t} \\
~\leq~& \brr{1 - \beta}^2 \norm{\lambda_t - \lambda^*}^2 + \eta_{\mathrm{max}} \bss{4 (1-\beta) z + 1} \Delta_{\mathrm{max}} ,
\end{align*}

where $(i)$ is by the property of projection, $(ii)$ is by applying Eq.~\eqref{proof-lemma-1-eq-1}, $(iii)$ is due to Cauchy–Schwarz inequality, $(iv)$ is by definition $\beta = \min_t \eta_t \beta_t$, and $(v)$ is due to the fact $\lambda \in \Omega = \bcc{\lambda: \norm{\lambda}_2 \leq z}$.

By using the fact that $\sqrt{a + b} \leq \sqrt{a} + \sqrt{b}$ for all positive $a,b$ (and by recurrence), we have:
\begin{align*}
\norm{\lambda_{t+1} - \lambda^*} ~\leq~& \brr{1 - \beta} \norm{\lambda_t - \lambda^*} + \sqrt{\eta_{\mathrm{max}} \bss{4 (1-\beta) z + 1} \Delta_{\mathrm{max}}} \\
~\leq~& \brr{1 - \beta}^t \norm{\lambda_1 - \lambda^*} + \sqrt{\eta_{\mathrm{max}} \bss{4 (1-\beta) z + 1} \Delta_{\mathrm{max}}} \sum_{s=0}^{\infty}{\brr{1 - \beta}^s} \\
~\leq~& \brr{1 - \beta}^t \norm{\lambda_1 - \lambda^*} + \sqrt{\eta_{\mathrm{max}} \bss{4 (1-\beta) z + 1} \Delta_{\mathrm{max}}} \cdot \frac{1}{\beta} \\
~\leq~& \frac{\epsilon'}{2} + \frac{\epsilon'}{2} ~=~ \epsilon' ,
\end{align*}
for $t = \brr{\log \frac{1}{1 - \beta}}^{-1} \log \frac{2 \norm{\lambda_1 - \lambda^*}}{\epsilon'}$, and $\Delta_{\mathrm{max}} = \frac{{\epsilon'}^2 \beta^2}{4 \eta_{\mathrm{max}} \bss{4 (1-\beta) z + 1}}$.
\end{proof}

\subsection{Proof of Lemma~\ref{smooth-lemma}}
\begin{proof}
We use the following inequalities in the proof:
\begin{itemize}
\item Pinsker's inequality: If $P$  and  $Q$ are two probability distributions on a measurable space, then
\begin{equation}
D_{\mathrm{TV}}(P,Q) ~\leq~ \sqrt{2D_{\mathrm{KL}}(P,Q)}
\label{eq:pinkers_inequality}
\end{equation}
\item Log-sum inequality: For non-negative $a_{1}, a_{2} ..., a_{n}$ and $b_{1}, b_{2} ..., b_{n}$ we have:
\begin{equation}
\sum_{i=1}^{n} a_{i}\log \frac{a_{i}}{b_{i}} ~\geq~ \left(\sum_{i=1}^{n} a_{i}\right) \log \frac{\sum_{i=1}^{n} a_{i}}{\sum_{i=1}^{n} b_{i}}
\label{eq:lo_sum_ineq}
\end{equation}
\item For any two policy $\pi$ and $\pi'$ (acting in the MDP $\mathcal{M}$), we have \cite[Lemma~A.1]{sun2018dual}:
\begin{equation}
D_{\mathrm{TV}}(\rho^{\pi}, \rho^{\pi'}) ~\leq~ \frac{2}{1-\gamma}  \cdot \max_{s} D_{\mathrm{TV}}(\pi(\cdot|s),\pi'(\cdot|s)) , 
\label{lips-mdp}
\end{equation}
where $D_{\mathrm{TV}}(p,q) := \sum_x \abs{p(x) - q(x)}$ is the TV-divergence between two distributions $p$ and $q$. This implies that ``similar" policies behave ``similarly" in the MDP. 
\end{itemize}
Then due to \eqref{lips-mdp}, and Pinsker's inequality~\eqref{eq:pinkers_inequality}, we have:
\begin{align}
D_{\mathrm{TV}} \brr{\rho^{\pi^L_\lambda},\rho^{\pi^L_{\lambda'}}} ~\leq~& \frac{2}{1-\gamma}  \cdot \max_{s} D_{\mathrm{TV}}\brr{\pi^L_\lambda(\cdot|s),\pi^L_{\lambda'}(\cdot|s)} 
\nonumber \\
~\leq~& \frac{2}{1-\gamma}  \cdot \sqrt{2 \max_{s} D_{\mathrm{KL}}\brr{\pi^L_\lambda(\cdot|s),\pi^L_{\lambda'}(\cdot|s)}} .\label{eq:delta_d_bound_eq_8}
\end{align}
The soft Bellman policy associated with the parameter $\lambda$ is given by:
\begin{align}
\pi^L_{\lambda} (a \mid s) ~=~& \frac{Z_{a \mid s , \lambda}}{Z_{s , \lambda}} \nonumber \\
\log Z_{s , \lambda}  ~=~& \log \sum_{a}{Z_{a \mid s , \lambda}} \nonumber \\
\log Z_{a \mid s , \lambda} ~=~& R^L_\lambda \brr{s,a} + \gamma \sum_{s'}{T(s' \mid s , a) \log Z_{s' , \lambda}}  .\label{eq:log_z_def_eq_1}
\end{align}
Consider for any state $s$:
\begin{align}
D_{\mathrm{KL}}(\pi^L_{\lambda}(\cdot|s),\pi^L_{\lambda'}(\cdot|s)) =  \sum_{a} {\pi^L_{\lambda}(a|s) \bcc{\log \frac{Z_{a \mid s , \lambda}}{Z_{a \mid s , \lambda'}} +\log \frac{Z_{s,\lambda'}}{Z_{s,\lambda}}}} \nonumber \\
= \left(\sum_{a}\pi^L_{\lambda}(a|s)\log \frac{Z_{a \mid s , \lambda}}{Z_{a \mid s , \lambda'}}\right) + \log \frac{Z_{s,\lambda'}}{Z_{s,\lambda}} ,
\label{eq:max_D_KL_eq_2}
\end{align}
where the first equation is by the definition of $D_{\mathrm{KL}}$, and second is due to the fact that $\sum_{a}\pi^L_{\lambda}(a|s) =1$. By using log-sum inequality~\eqref{eq:lo_sum_ineq} we have:
\begin{align}
\log \frac{Z_{s,\lambda'}}{Z_{s,\lambda}} ~=~& \underbrace{\frac{1}{Z_{s,\lambda'}} \cdot \sum_{a} {Z_{a|s,\lambda'} \cdot}}_{=1} \log \frac{\sum_{a} Z_{a|s,\lambda'}}{\sum_{a} Z_{a|s,\lambda}} \nonumber \\
~\leq~& \frac{1}{Z_{s,\lambda'}}\sum_{a}{Z_{a|s,\lambda'} \log \frac{Z_{a|s,\lambda'}}{Z_{a|s,\lambda}}} \nonumber \\
~=~& \sum_{a}\pi^L_{\lambda'}(a|s) \log \frac{Z_{a|s,\lambda'}}{Z_{a|s,\lambda}}
\label{eq:log_z_s_lambda_eq_3}
\end{align}
From \eqref{eq:max_D_KL_eq_2} and \eqref{eq:log_z_s_lambda_eq_3} we have:
\begin{align}
D_{\mathrm{KL}}(\pi^L_{\lambda}(\cdot|s),\pi^L_{\lambda'}(\cdot|s)) ~\leq~& \sum_{a} \left(\pi^L_{\lambda}(a|s)-\pi^L_{\lambda'}(a|s)\right) \log \frac{Z_{a \mid s , \lambda}}{Z_{a \mid s , \lambda'}} \nonumber \\
~\leq~& \sum_{a} \abs{\pi^L_{\lambda}(a|s)-\pi^L_{\lambda'}(a|s)} \cdot \abs{\log \frac{Z_{a \mid s , \lambda}}{Z_{a \mid s , \lambda'}}} \nonumber \\
~\leq~& \sum_{a} \abs{\pi^L_{\lambda}(a|s)-\pi^L_{\lambda'}(a|s)} \cdot \max_{a} \abs{\log \frac{Z_{a \mid s , \lambda}}{Z_{a \mid s , \lambda'}}} \nonumber \\
~\leq~& \max_{a} \abs{\log \frac{Z_{a \mid s , \lambda}}{Z_{a \mid s , \lambda'}}} D_{\mathrm{TV}}(\pi^L_{\lambda}(\cdot|s),\pi^L_{\lambda'}(\cdot|s)) \nonumber \\
~\leq~& \max_{a} \abs{\log \frac{Z_{a \mid s , \lambda}}{Z_{a \mid s , \lambda'}}} ,
\label{eq:bound_D_KL_pi_eq_4}
\end{align}
where the last inequality is due to the fact that $D_{\mathrm{TV}}(\pi^L_{\lambda}(\cdot|s),\pi^L_{\lambda'}(\cdot|s)) \leq 1$. Let
\[
(a^*,s^*) ~:=~ \argmax_{a,s} \abs{\log \frac{Z_{a|s,\lambda}}{Z_{a|s,\lambda'}}}
\]
Then we have
\begin{align}
\log \frac{Z_{a^*|s^*,\lambda}}{Z_{a^*|s^*,\lambda'}} ~=~& \log {Z_{a^*|s^*,\lambda}} - \log{Z_{a^*|s^*,\lambda'}} \nonumber \\
~=~& R^L_\lambda \brr{s^*,a^*} - R^L_{\lambda'} \brr{s^*,a^*} + \gamma \sum_{s'} {T(s'|s^*,a^*)\log \frac{Z_{s',\lambda}}{Z_{s',\lambda'}}}
\nonumber \\
~\leq~& R^L_\lambda \brr{s^*,a^*} - R^L_{\lambda'} \brr{s^*,a^*} + \gamma \sum_{s'} {T(s'|s^*,a^*)\sum_{a}{\pi^L_{\lambda}(a|s') \log \frac{Z_{a|s',\lambda}}{Z_{a|s',\lambda'}} }} , 
\label{eq:log_z_s_a_lambda_bound_eq_5}
\end{align}
where the last inequality follows by argument similar to \eqref{eq:log_z_s_lambda_eq_3}. Now consider:
\begin{align*}
\abs{\log \frac{Z_{a^*|s^*,\lambda}}{Z_{a^*|s^*,\lambda'}}} ~\leq~& \abs{R^L_\lambda \brr{s^*,a^*} - R^L_{\lambda'} \brr{s^*,a^*} + \gamma \sum_{s'} {T(s'|s^*,a^*)\sum_{a}{\pi^L_{\lambda}(a|s') \log \frac{Z_{a|s',\lambda}}{Z_{a|s',\lambda'}}}}} \\
~\leq~& \abs{R^L_\lambda \brr{s^*,a^*} - R^L_{\lambda'} \brr{s^*,a^*}} + \gamma \sum_{s'} {T(s'|s^*,a^*)\sum_{a}{\pi^L_{\lambda}(a|s') \abs{\log \frac{Z_{a|s',\lambda}}{Z_{a|s',\lambda'}}} }} \\
~\leq~& \abs{R^L_\lambda \brr{s^*,a^*} - R^L_{\lambda'} \brr{s^*,a^*}} + \gamma \sum_{s'} {T(s'|s^*,a^*)\sum_{a}{\pi^L_{\lambda}(a|s') \abs{\log \frac{Z_{a^*|s^*,\lambda}}{Z_{a^*|s^*,\lambda'}}} }} \\
~=~& \abs{R^L_\lambda \brr{s^*,a^*} - R^L_{\lambda'} \brr{s^*,a^*}} + \gamma \abs{\log \frac{Z_{a^*|s^*,\lambda}}{Z_{a^*|s^*,\lambda'}}} \sum_{s'} {T(s'|s^*,a^*)\sum_{a}{\pi^L_{\lambda}(a|s')}} \\
~=~& \abs{R^L_\lambda \brr{s^*,a^*} - R^L_{\lambda'} \brr{s^*,a^*}} + \gamma \abs{\log \frac{Z_{a^*|s^*,\lambda}}{Z_{a^*|s^*,\lambda'}}} \\
~\leq~& \max_{s,a} \abs{R^L_\lambda \brr{s,a} - R^L_{\lambda'} \brr{s,a}} + \gamma \abs{\log \frac{Z_{a^*|s^*,\lambda}}{Z_{a^*|s^*,\lambda'}}}
\end{align*}
Hence,
\begin{align}
(1-\gamma) \abs{\log \frac{Z_{a^*|s^*,\lambda}}{Z_{a^*|s^*,\lambda'}}} ~\leq~& \max_{s,a} \abs{R^L_\lambda \brr{s,a} - R^L_{\lambda'} \brr{s,a}} .
\label{eq:delta_d_bound_eq_6}
\end{align}
Considering \eqref{eq:bound_D_KL_pi_eq_4} and \eqref{eq:delta_d_bound_eq_6} we have:
\begin{align}
\max_s D_{\mathrm{KL}}(\pi^L_{\lambda}(\cdot|s),\pi^L_{\lambda'}(\cdot|s)) ~\leq~& \frac{1}{1-\gamma} \max_{s,a} \abs{R^L_\lambda \brr{s,a} - R^L_{\lambda'} \brr{s,a}} .
\label{eq:delta_d_bound_eq_7}
\end{align}
Combining \eqref{lips-reward}, \eqref{eq:delta_d_bound_eq_8}, and \eqref{eq:delta_d_bound_eq_7} completes the proof. The second bound follows directly from the following inequality:
\begin{align}
\abs{\nu^{\pi^L_{\lambda}} - \nu^{\pi^L_{\lambda'}}} ~=~& \frac{1}{1-\gamma} \abs{\sum_{s,a}{\bcc{\rho^{\pi^L_{\lambda}}(s,a)-\rho^{\pi^L_{\lambda'}}(s,a)} R^E(s,a)}} \nonumber \\ 
~\leq~& \frac{1}{1-\gamma} \sum_{s,a}{\abs{\rho^{\pi^L_{\lambda}}(s,a)-\rho^{\pi^L_{\lambda'}}(s,a)} \abs{R^E(s,a)}} 
\nonumber \\
~\leq~& \frac{1}{1-\gamma} \sum_{s,a}{\abs{\rho^{\pi^L_{\lambda}}(s,a)-\rho^{\pi^L_{\lambda'}}(s,a)} R^E_{\mathrm{max}}} 
\nonumber \\
~=~& \frac{R^E_{\mathrm{max}}}{1-\gamma} D_{\mathrm{TV}} (\rho^{\pi^L_{\lambda}},\rho^{\pi^L_{\lambda'}}) \nonumber 
\end{align}
\end{proof}

\section{Computing Target Hyperparameter $\lambda^*$ for Linear Reward Function}\label{sec.appendix.policy-hyper-teaching}

In this section, we consider a learner model with linear reward function $R^L_\lambda (s,a) = \ipp{\lambda}{\phi^L(s,a)}$, and teacher with linear reward function $R^E (s,a) = \ipp{w^E}{\phi^L(s,a)}$.  For this case, we show that there exists a $\lambda^*$ such that $\abs{\nu^{\pi^L_{\lambda^*}} - \nu^{\pi^E}} ~\leq~ \frac{\epsilon}{2}$, and that $\lambda^*$ can be computed efficiently.\footnote{Note that the learner uses soft Bellman policy, and using $\lambda = w^E$ might not satisfy  $\abs{\nu^{\pi^L_{\lambda}} - \nu^{\pi^E}} ~\leq~ \frac{\epsilon}{2}$.} 

\paragraph{Existence of $\lambda^*$}
Consider the following optimization problem:
\[
\max_\lambda ~ \underset{\bcc{\brr{s_\tau , a_\tau}}_\tau \sim \brr{\pi^E , \mathcal{M}}}{\mathbb{E}} \bss{\sum_\tau \gamma^\tau \log \pi^L_\lambda \brr{a_{\tau} \mid s_{\tau}}}
\]
\cite{ziebart2010modeling} have shown that the above optimization problem has a unique solution $\lambda^\mathrm{opt}$, and it satisfies $\nu^{\pi^L_{\lambda^\mathrm{opt}}} = \nu^{\pi^E}$.

\paragraph{Computation of $\lambda^*$}
Proposition~\ref{mce-prop-linear} provides a constructive way of obtaining such $\lambda^*$ with high probability. By invoking the Proposition~\ref{mce-prop-linear} with $\tilde \epsilon = \frac{(1-\gamma) \epsilon}{2 \norm{w^E}}$, with probability at least $1-\delta$, we get  the following:
    \begin{equation}
    \label{lin-final-proo-1}
    \abs{\nu^{\pi^L_{\lambda^*}} - \nu^{\pi^E}} ~\leq~ \frac{\norm{w^E}}{1-\gamma} \cdot \tilde \epsilon ~=~ \frac{\epsilon}{2} .
    \end{equation}

\begin{proposition}
Given $\tilde \epsilon > 0$ and $\delta > 0$, let $\Xi^E = \bcc{\xi_t}_{t=1,2,\dots,m}$ (where $\xi_t = \bcc{\brr{s_{t,\tau}, a_{t,\tau}}}_{\tau=0,1,\dots}$) be a collection of $m$ demonstrations generated by following the policy $\pi^E$ in the MDP $\mathcal{M}$ starting from $s \sim P_0$. Here $m \geq \frac{2 d}{\tilde \epsilon^2} \log \frac{2 d}{\delta}$, and the demonstrations are truncated at length $H = \log_{\gamma} \bss{\frac{\tilde \epsilon}{2 \sqrt{d}}}$. Define 
\begin{equation}
\label{app-target-hyper-eq}
\lambda^* ~=~ \argmax_\lambda ~~ \sum_{t=1}^m \sum_\tau \gamma^\tau \log \pi^L_\lambda \brr{a_{t,\tau} \mid s_{t,\tau}} .
\end{equation}
Then with probability at least $1-\delta$, we have:
\[
\abs{\nu^{\pi^L_{\lambda^*}} - \nu^{\pi^E}} ~\leq~  \frac{\norm{w^E}}{1-\gamma} \cdot \tilde \epsilon .
\]
\label{mce-prop-linear}
\end{proposition}

\begin{proof}
First note that
\[
\norm{\mu^{\pi^L_{\lambda^*}} - \mu^{\pi^{E}}} ~\leq~ \norm{\mu^{\pi^L_{\lambda^*}} - \mu^{\Xi^{E}}} + \norm{\mu^{\Xi^{E}} - \mu^{\pi^E}} ,
\]
where $\mu^{\pi} = \sum_{s,a} \rho^{\pi} (s,a) \cdot \phi^L (s,a)$, and $\mu^{\Xi^E} = \sum_{s,a} \rho^{\Xi^E} (s,a) \cdot \phi^L (s,a)$. The proof completes by combining the following results:

\begin{enumerate}[(i)]
    \item \cite{ziebart2010modeling} have shown that $\lambda^*$ given in Eq.~\eqref{app-target-hyper-eq} exists, and that $\lambda^*$ satisfies $\mu^{\pi^L_{\lambda^*}} = \mu^{\Xi^{E}}$. 
    \item \cite{abbeel2004apprenticeship} have shown that the construction scheme of the demonstration set $\Xi^E$ satisfies $\norm{\mu^{\Xi^{E}} - \mu^{\pi^E}} \leq \tilde \epsilon$.
    \item $\abs{\nu^{\pi^L_{\lambda^*}} - \nu^{\pi^E}} \leq \frac{1}{1-\gamma} \cdot \norm{w^E} \cdot \norm{\mu^{\pi^L_{\lambda^*}} - \mu^{\pi^E}}$.
\end{enumerate}
\end{proof}

}
}
{}
\end{document}